\documentclass{article}
\usepackage{ijcai20}

\usepackage{times}
\usepackage{soul}
\usepackage{url}
\usepackage[hidelinks]{hyperref}
\usepackage[utf8]{inputenc}
\usepackage[small]{caption}
\usepackage{graphicx}
\usepackage{amsmath}
\usepackage{amsthm}
\usepackage{booktabs}
\usepackage{algorithm}
\usepackage{algorithmic}
\graphicspath{ {figures/} }
\usepackage{bbm}
\usepackage{xcolor}

\usepackage{float}
\usepackage{wrapfig}

\newtheorem{thm}{Theorem}
\newtheorem{lem}{Lemma}

\newtheorem{assum}{Assumption}

\newtheorem{proposition}{Proposition}

\newtheorem{remark}{Remark}
\newcommand{\argmax}{\operatornamewithlimits{argmax}}

\title{Multiscale Non-stationary Stochastic Bandits}

\author{
Qin Ding$^1$
\and
Cho-Jui Hsieh$^2$\and
James Sharpnack$^{1}$
\affiliations
$^1$ Department of Statistics, University of California, Davis \\
$^2$ Department of Computer Science, University of California, Los Angeles 
\emails
qding@ucdavis.edu,
chohsieh@cs.ucla.edu,
jsharpna@ucdavis.edu.
}

\begin{document}

\maketitle

\begin{abstract}
  Classic contextual bandit algorithms for linear models, such as LinUCB, assume that the reward distribution for an arm is modeled by a stationary linear regression. When the linear regression model is non-stationary over time, the regret of LinUCB can scale linearly with time. In this paper, we propose a novel multiscale changepoint detection method for the non-stationary linear bandit problems, called Multiscale-LinUCB, which actively adapts to the changing environment. We also provide theoretical analysis of regret bound for Multiscale-LinUCB algorithm. Experimental results show that our proposed Multiscale-LinUCB algorithm outperforms other state-of-the-art algorithms in non-stationary contextual environments.
\end{abstract}

\section{Introduction}

The multi-armed bandit (MAB) problem is a sequential learning setting, where each round the player decides which arm to pull from a $K$-arm bandit. The player only observes partial reward feedback according to the pulled arm and may use the past rewards to adapt its strategy. The goal is to balance the trade-off between exploration and exploitation over time and minimize the cumulative regret up to $T$ rounds.
The MAB setting, first introduced by \cite{first_mab}, has received extensive research during the past few decades due to its significant applications to online advertisements \cite{online_ads} and recommender systems \cite{news_contextual,liu2018information}. 
Recently, the contextual bandit setting \cite{first_context,develop_context} receives increasing interest
due to its efficiency in the case of large recommender system ($K$ is large) and interrelated reward distributions. Linear Upper Confidence Bound algorithm (LinUCB) \cite{news_contextual} was proposed for contextual bandit setting under a linear assumption, where the reward of each arm is predicted by a linear model of feature vectors $x\in\mathbbm{R}^p$ and linear regression parameter $\theta \in \mathbbm{R}^p$. This is known as the stochastic linear bandits. Chu et al. \shortcite{linucb} proved a lower bound of $O(\sqrt{Tp})$ for linear bandit setting, where $p$ is the dimension of feature vectors. It was shown that LinUCB can achieve this lower bound with a logarithm factor \cite{linucb}.

Most existing stochastic linear bandit algorithms like LinUCB and Linear Thompson Sampling (LinTS) \cite{lints}
assume the regression parameters for rewards stay non-stationary over time. 
However, in reality, the assumption of stationarity rarely holds. As an example, in news recommendation, a user might be more interested in political news during the presidential debate, and more interested in sports news during the NBA playoff season. Popular algorithms like LinUCB or LinTS which achieve optimal regret bounds in stationary environments could end up with linear regret for non-stationary environments in the worst case. Many efforts have been taken to emphasize this problem \cite{cheung2018learning,cheung2018hedging,russac2019weighted,wu2018learning}, including methods of passively and actively adapting to the changing environment.

We explore the solutions for piecewise-stationarity in stochastic bandit settings with linear assumptions, where the regression model parameter stays stationary for a while and changes abruptly at a certain time. The main idea is to design a changepoint detection method and perform the classic LinUCB algorithm within the intervals of homogeneity. When we detect a changepoint for an arm, we reset the LinUCB index for this arm. While the changepoint-based method sounds reasonable, it hasn't been successful due to the extreme difficulty of detecting faint changes in bandit problems. Piecewise-stationary environment in previous works mostly assumes the change in mean reward (at least for some portion of the arms) is bounded below by a constant \cite{wu2018learning}. However, faint changes are hardly ignorable. For example, neglecting to pull an optimal arm with faint changes over a stationary window of length $\Omega(T)$ is going to incur a large regret. 

In this paper, we first propose a piecewise-stationary environment with weaker assumptions, where we do not need the change in mean reward to be bounded below. We only require that for small changes, the consecutive stationary periods should be relatively long enough for our algorithm to detect a change and vice versa. We then propose a multiscale changepoint detection based algorithm, Multiscale-LinUCB, for the piecewise-stationary linear bandit setting (formally defined in Section \ref{pse}) and prove the nearly optimal regret bound for this algorithm. We show that the multiscale nature of the changepoint detector is essential for preventing poor regret when there are faint changes in reward distribution. Then we extend this setting to piecewise-stationary MAB bandit setting, where the reward distributions of some arms may change at certain changepoints. Extensive research in experiments show that our algorithm performs significantly better than other state-of-the-art algorithms in non-stationary environment.

\paragraph{Related Works: }
There is an important line of work for non-stationary MAB problems \cite{dis-sw,cusum,peter_new,besbes2014stochastic,mucb}. 
Recently, there has also been some novel researches that consider non-stationary contextual (can be non-linear) bandit algorithms where there are probabilistic assumptions on the context vectors \cite{chen2019new,luo2017efficient}. Chen et al. \shortcite{chen2019new} attains parameter-free and efficient algorithm assuming access to an 
ERM oracle. Here, we will only discuss some previous works on stochastic linear bandit algorithms for non-stationary environments 
, as those works are closely related to ours.

The recently developed D-LinUCB \cite{russac2019weighted} employs a weighted linear bandit model, where the weight is adjusted according to how recently the data point is observed. By putting a discount rate on past observations for computing the LinUCB index, it passively adapts to the changing environment. This work has its similarities in Discounted UCB, which is proposed for non-stationary MAB \cite{dis-sw}. In the same work, Garivier and Moulines \shortcite{dis-sw} proposed Sliding Window UCB for non-stationary MAB. Cheung et al. \shortcite{cheung2018learning} generalized it to Sliding Window LinUCB (SW-LinUCB) for non-stationary stochastic linear bandit. SW-LinUCB computes the LinUCB index according to the most recent $\tau$ observations, where $\tau$ is the sliding window size. Both D-LinUCB and SW-LinUCB assumes the knowledge of the total variation bound $B_T$ where $\sum_{t=1}^{T-1} \|\theta_{t+1}-\theta_t\|_2 \leq B_T$, which is rarely practical in reality. Here $\theta_t$ is the true model parameter for the regression model at time $t$. When the discount rate of D-LinUCB or the window size of SW-LinUCB is chosen based on $B_T$, both algorithms can attain a regret upper bound of $O\left((pT)^{\frac{2}{3}} (B_T)^{\frac{1}{3}}\right)$. 

In addition to passively adapting to the changing environment, there has also been substantial works considering actively adapting to the changing environments by changepoint detection methods. 
These works are mostly proposed for piecewise-stationary environment, and most of them assume the change in reward is bounded from below. 
The idea can track back to many algorithms in piecewise-stationary MAB environment \cite{cusum,mucb,pht}. 
For piecewise-stationary linear bandits, Wu et al. \shortcite{wu2018learning} proposed Dynamic Linear UCB (dLinUCB) algorithm. The key idea of dLinUCB is to maintain a master bandit model which keeps track of the ``badness" of several slave bandit models. The best slave model is chosen to determine which arm to pull each time and the feedback is shared with all models in the system. When there is no ``good" slave model in the pool, a change is detected and a new slave model is created. Wu et al. \shortcite{wu2018learning} showed that when the ``badness" of the model is set based on the proportion of arms changing and the lower bound of changes in rewards, then the algorithm can attain an optimal regret upper bound of $O(D\sqrt{S_{\max}}\log S_{\max})$, where $S_{\max}$ is the length of the longest stationary period and $D$ is the total number of changepoints. 

\section{Methodology}
\subsection{Problem Formulation}\label{setup}
We consider the contextual bandit problems with disjoint linear models proposed by Li et al. \shortcite{news_contextual} in non-stationary environment.
In a time horizon $[1,T]$, let $\mathcal{K} = \{1,\dots, K\}$ be the set of arms. At time $t$, the player has access to the feature vectors of every arm $\mathcal{A}_t = \{ x_{t,1}, x_{t,2}, \dots, x_{t,K}\}\subset \mathbbm{R}^p$. After observing $\mathcal{A}_t$, the player chooses an action $I_t \in \mathcal{K}$ and observe a sample reward $y_{t,I_t}$. Each time, the observed reward is independent of each other.

In the stationary setting, expected reward of arm $i$ at time $t$ is modeled as a function of unknown vector $\theta_i \in \mathbbm{R}^p$ and feature vectors. Under linear assumption, the expected reward becomes
\begin{equation}
\tag{stationary}
    \mu_{t,i} = E[y_{t,i} | x_{t,i}] = x_{t,i}^T \theta_i. 
\end{equation}

In non-stationary contextual setting, $\theta_i$ could change over time. We assume that for arm $i$ there are in total $\gamma_i+1$ changepoints, denoted as $\mathcal{C}_i = \{c_{i,0}, c_{i,1} \dots, c_{i,\gamma_i}\}$, where $c_{i,0}=1$ and $c_{i,\gamma_i} = T+1$. 
We say that $c_{i,j}$ is a changepoint for arm $i$ if the model parameter $\theta_i$ is different before and after time $c_{i,j}$.
Specifically, for the $\gamma_i$ stationary periods, we define the length of the $j$-th stationary period to be $S_{i,j} = c_{i,j} - c_{i,j-1}$ and are associated with an unknown parameter $\theta_{i,j} \in \mathbbm{R}^p$, where $j = 1, \dots, \gamma_i$. We have
\begin{equation}\label{model}
    \mu_{t,i} = x_{t,i}^T \theta_{i,j},  c_{i,j-1} \le t < c_{i,j}. 
\end{equation}

Define $\mathcal{C} = \displaystyle{\cup_{i=1}^K} \mathcal{C}_i$ and $D = |\mathcal{C}|$. Then $\mathcal{C}$ will be the set of all changepoints and $D$ is the total number of changepoints.
Note that it is possible that $\sum_{i=1}^K \gamma_i > D$, which means that there are multiple arms changing at the same time.
See Figure \ref{notation} for illustration of the notations.

\begin{figure}[H]
 \vskip -0.1in
  \begin{center}
    \includegraphics[width=0.9\columnwidth]{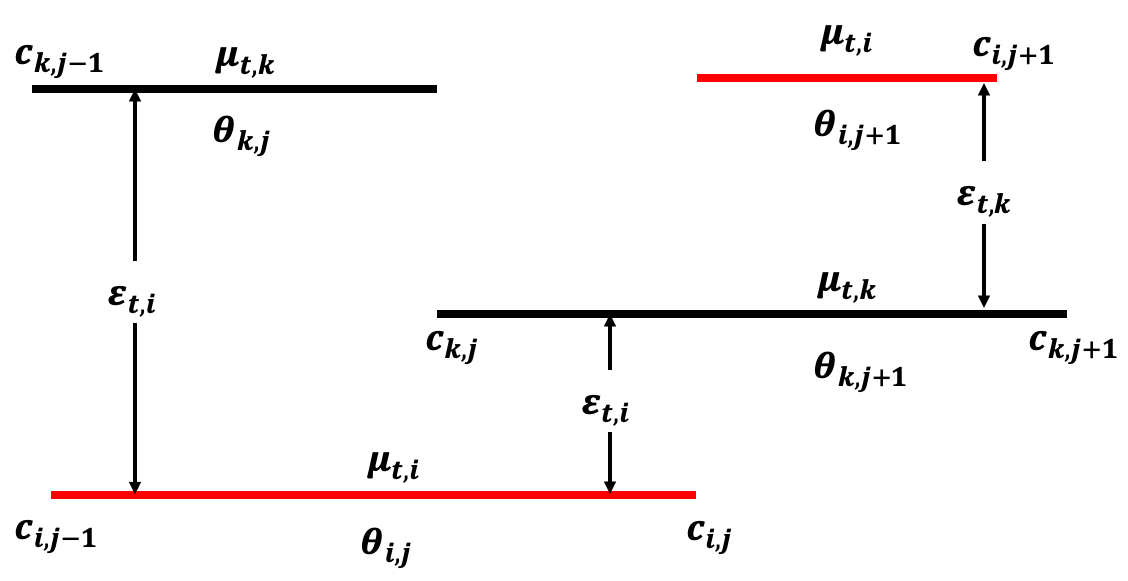}
  \end{center}
  \vskip -0.2in
  \caption{Illustration of notations. Black solid line represents arm $k$, red solid line represents arm $i$. At first, arm $k$ is the optimal arm, after changepoint $c_{i,1}$, arm $i$ becomes the optimal arm. }
  \label{notation}
   \vskip -0.1in
\end{figure}

Define the optimal arm at time $t$ to be $a_t$, i.e., $\mu_{t,a_t} = \max_{i=1}^K \mu_{t,i}$, where $\mu_{t,i}$ is defined in Equation \ref{model}. Also define $\epsilon_{t,i} = \mu_{t,a_t} - \mu_{t,i}$.
Similar to stationary settings, the goal of the decision maker is to find a policy $\pi$, so that following policy $\pi$, it chooses an arm $I_t$ every time to minimize the total regret over time, where
the total regret is defined to be
\begin{equation*}
    R_{\pi}(T) = E\left[ \sum_{t=1}^T (y_{t,a_t}, - y_{t,I_t}) \right]. 
\end{equation*}


\subsection{Piecewise-stationary Environment}\label{pse}
We study the piecewise-stationary environment in \cite{sideinfo}, where the reward distribution remains the same for a while and abruptly changes at a changepoint. 
In addition, we propose two mild assumptions for our piecewise-stationary contextual environment.
\begin{assum}\label{sub_g}
\textbf{(Sub-Gaussian Reward)}
The reward distribution is sub-Gaussian with parameter $\sigma^2$, without loss of generality, we assume $\sigma^2=1$ for the analysis below.
\end{assum}
Assumption \ref{sub_g} has been widely used in the literature. It includes the widely used Bernoulli reward in online recommender systems.

\begin{assum}\label{ass_context}
\textbf{(Detectability)}
There exists a constant $C^\prime$
such that the following holds.
For arm $i$ and adjacent stationary periods $j,j+1$ of length $S_{i,j}$ and $S_{i,j+1}$ respectively, true parameter $\theta_i$ changes from $\theta_{i,j}$ to $\theta_{i,j+1}$,
and for any $t$ in these two stationary periods,
define $\delta_{i,j} = \min_{t\in [c_{i,j-1}, c_{i,j+1})} \|x_{t,i}^T (\theta_{i,j} -\theta_{i,j+1})\|$. We assume the following inequalities hold.
\begin{align*}
S_{i,j} \delta_{i,j}^2  &\geq 4C^\prime p \sqrt{T\log T}, \\
S_{i,j+1} \delta_{i,j}^2  &\geq 4C^\prime p \sqrt{T\log T}.
\end{align*}
\end{assum}

Assumption \ref{ass_context} is weaker than most of the assumptions made in literature \cite{cusum,wu2018learning}. Most changepoint-based method for piecewise-stationary bandit assumes $\delta_{i,j}$ bounded below to ensure detectability. However, our method does not need this. Assumption \ref{ass_context} means that when $\delta_{i,j}$ is small, we need longer stationary periods $S_{i,j}$ and $S_{i,j+1}$ for us to detect a changepoint. For example, this condition allows stationary periods of length $\Omega(T)$ with faint changes $\delta_{i,j} \geq O(\sqrt[4]{\log T / T})$.

\subsection{Proposed Algorithm: Multiscale-LinUCB}
In this section, we introduce our proposed changepoint detection based LinUCB algorithm, Multiscale-LinUCB. 
Generally speaking, the algorithm performs LinUCB algorithm when there is no changepoint, and when we detect a changepoint for an arm, we reset the LinUCB index for this arm. 

One of the biggest challenges for changepoint detection in the stochastic bandit setting is that LinUCB will not pull every arm frequently enough to detect a change in reward distribution. Due to the nature of LinUCB, it will eventually stop pulling suboptimal arms, but this can cause a missed changepoint in this arm. If this arm then becomes optimal, this new optimal arm will continue to be neglected, resulting in a regret that is linear with $T$.
To remedy this problem, we randomly preselect some ``changepoint detection'' rounds with probability $\alpha := \sqrt{\frac{\log T}{T}}$ to pull arm $i$. These are rounds at which we pull an arm purely for the purpose of detecting changepoints.
Therefore, for each arm $i$, there will be approximately $\lceil \alpha T \rceil$ preselected rounds, $\bar B_i \subseteq\{1,\ldots,T\}$ such that $\{\bar B_i\}$ are disjoint. This probability $\alpha$ is carefully selected so that we can balance between minimizing total regret and the need of having enough samples for detecting changes in every arm. 
Moreover, in non-stationary bandit settings, there could be a changepoint at any time, so it is important to maintain at least some level of exploration all the time to make sure that we still have a chance to choose the optimal arm at current time, even though this optimal arm could be the worst arm in previous times. 

Let's focus on a single arm $i$ now. Assume we have detected the most recent changepoint $c_{i,\gamma}$, are now at time $t$, and for any cut point $t^{\prime} \in (c_{i,\gamma}, t]$, we cut the interval into the two parts. Define $B_1 = [c_{i,\gamma}, t^{\prime}) \cap \bar B_i$ and $B_2 = [t^{\prime}, t] \cap \bar B_i$. 
Over $B_1 \cup B_2$, we get observed rewards $\tilde y = (\tilde y_1^T, \tilde y_2^T)^T$, where $\tilde y_1 \in \mathbbm{R}^{|B_1|}$ has its elements as $y_{t,i}$, $t\in B_1$. Similarly, $\tilde y_2 \in \mathbbm{R}^{|B_2|}$ has its elements as $y_{t,i}$, where $t\in B_2$. Define the design matrix for arm $i$ at interval $B_1$ as $\tilde X_1 \in \mathbbm{R}^{|B_1|\times p}$. Each row of $\tilde X_1$ is $x_{t,i}$ where $t\in B_1$. We similarly have $\tilde X_2, \tilde X$ defined as the design matrices for arm $i$ at intervals $B_2$ and $B_1 \cup B_2$.


\begin{algorithm}[H]
\caption{Multiscale changepoint detection} 
\label{alg:cp}
\textbf{Input}: Arm index $i$, current time $t$, the most recent detection point $c_{i,\gamma}$.
\begin{algorithmic}[1]
   \FOR{$t^{\prime} = c_{i,\gamma}+1$ {\bfseries to} $t$}
        \STATE Calculate $Z_{i,t,t^{\prime}}$ according to Equation \ref{eq:context_stat}
        \IF{$Z_{i,t,t^{\prime}}^2 \geq Cp\log T$}
            \STATE Return changepoint $t$
            \STATE Break
        \ENDIF
   \ENDFOR
\end{algorithmic}
\end{algorithm}

Since OLS estimator is an unbiased estimator for $\theta_{i,j}$, to detect changepoints for arm $i$, we calculate the OLS estimators for intervals $B_1, B_2$ and $B_1 \cup B_2$ as follows:
\begin{align}
\label{eq:theta_one}
\hat{\theta_1} &= (\tilde X_1^T \tilde X_1)^{-1} \tilde X_1^T \tilde y_1, \\
\label{eq:theta_two}
\hat{\theta_2} &= (\tilde X_2^T \tilde X_2)^{-1} \tilde X_2^T \tilde y_2, \\
\label{eq:theta}
\hat{\theta} &= (\tilde X^T \tilde X)^{-1} \tilde X^T \tilde y. 
\end{align}

We claim there is a changepoint at time $t$ for arm $i$ if there exists a $t^\prime \in (c_{i,\gamma},t]$ such that $Z_{i,t,t'}^2 \geq C p\log T$, where $C$ is a constant to be specified and $Z_{i,t,t'}^2$ defined as, 
\begin{equation}
\label{eq:context_stat}
    Z_{i,t,t'}^2 := \| \tilde X_1(\hat\theta_1-\hat\theta) \|^2 + \| \tilde X_2 (\hat \theta_2 - \hat\theta) \|^2.
\end{equation}
Otherwise, we assert that there is no changepoint in interval $[c_{i,\gamma}, t]$. See Algorithm \ref{alg:cp} for details.
In our algorithm, we will need to check the following condition in order to verify the trustworthiness of our detection, and we also require this condition for the true changepoints.

\begin{assum}\label{length_context}
\textbf{(Minimum stationary length and well conditioned)}
There exists some universal constant $\xi \in (1,2)$, such that for every arm $i$, and two adjacent stationary regions, if we compute $B_1,B_2$ and $\tilde X_1,\tilde X_2$ then the following hold:
    $$
    \min \{ |B_1|, |B_2| \} \geq p 
    $$
    \text{and such that}
    $$
    \xi \tilde \Sigma_2 \succeq \tilde \Sigma_1 \succeq \xi^{-1} \tilde \Sigma_2.
    $$
    Here $\Sigma_j$ is the Gram matrix defined as 
\[
\tilde \Sigma_j := \frac{1}{|B_j|} \tilde X_j^\top \tilde X_j, \quad j=1,2.
\]
\end{assum}
Notice that Assumption \ref{length_context} is check-able for proposed $\xi, i, B_1, B_2$. 
Moreover, we have Proposition \ref{prop:assum4} below showing that Assumption \ref{length_context} is valid under many circumstances. 
\begin{proposition}\label{prop:assum4}
If matrix $\tilde X_1$ and $\tilde X_2$ have independent sub-Gaussian rows with the same second moment matrix, also assume the consecutive stationary periods are of length $S_1, S_2$ respectively and $S_1, S_2 = \omega(\sqrt{T\log T})$, we have Assumption \ref{length_context} holds with probability greater than $1-\frac{3}{T}$.
\end{proposition}

Our proposed Multiscale-LinUCB algorithm is formally presented in Algorithm \ref{alg:cp_ucb_context}. Our analysis is only valid when the preselected rounds, $B_1,B_2$ can be considered fixed (or predetermined independently of the data). In practice, we will combine these rounds with those sampled from the LinUCB steps as well for the changepoint detection steps.
We would like to clarify that Assumption \ref{ass_context}, \ref{length_context} are only needed in theoretical analysis. In practice, our proposed Multiscale-LinUCB can achieve significantly better experimental results even if these two assumptions do not hold in some settings, as shown in Section \ref{exps}.

\begin{algorithm}[H]
\caption{Multiscale-LinUCB for contextual bandit} 
\label{alg:cp_ucb_context}
\textbf{Input}: Input: $\xi \in (1,2), C > 0$ 
\begin{algorithmic}[1]
\STATE For each arm, $i$, randomly preselect $\lceil \alpha T \rceil$ times, $\bar B_i \subseteq\{1,\ldots,T\}$ such that $\{\bar B_i\}$ are disjoint. 
\STATE Initialize $c_{i,\gamma} = 1$, $\theta_i = 0$ for all $i=1,\dots, K$.
    \FOR{$t = 1$ {\bfseries to} $T$}
        \IF {$t \in \bar B_i$ for some $i$}
            \STATE Pull arm $i$ 
            \STATE Observe $y_{t,i}, x_{t,i}$ and add to $\tilde y$, $\tilde X$ respectively
            \STATE Run Algorithm \ref{alg:cp} with input $i,t, c_{i,\gamma}$.
            \IF {Algorithm \ref{alg:cp} returns a changepoint and Assumption \ref{length_context} holds with $\xi$}
                \STATE Update $c_{i,\gamma} \gets t$
                \STATE Reset $\theta_i \gets 0$ in LinUCB.
            \ENDIF
        \ELSE
            \STATE Run steps of LinUCB
        \ENDIF
    \ENDFOR
 \end{algorithmic}
 \end{algorithm}

\section{Analysis}\label{analysis_context}
If Algorithm \ref{alg:cp} can detect every changepoint successfully, then we can just restart the LinUCB algorithm at the beginning of every stationary period to achieve a regret upper bound of $D R_{\text{LinUCB}}(S_{\max}) = \tilde O(D\sqrt{S_{\max}})$, where $S_{\max}$ is defined to be the length of the longest stationary period and $R_{\text{LinUCB}}(S)$ is the regret of LinUCB in a stationary period of length $S$. However, for every changepoint detection method, there will be false alarms and detection delays. Assume we are at time $t$, false alarm means that even though there is no changepoint in interval $[c_{i,\gamma}, t]$, the algorithm alarms us that there is a changepoint at time $t$. 
For changepoint $c_{i,\gamma}$, if the algorithm alarms us at time $t > c_{i,\gamma}$, then the detection delay is defined to be $t-c_{i,\gamma}$. 

The following lemma controls the probability of missed changepoints when the sampled stationary regions $B_1,B_2$ satisfy our Assumption \ref{length_context} and Equation \ref{important_contextual_assum}.
\begin{lem}\label{delay_context}
Let $\xi \in (1,2)$.
Consider all adjacent sampled regions $B_1,B_2$ that satisfy 

    $(1).$ Assumptions \ref{length_context} holds;
    
    $(2).$ 
For an arm $i$, and two adjacent stationary regions, $j,j+1$, where $\theta_{i,j}$ changes to $\theta_{i,j+1}$, if we compute $\tilde X$ over $B_1 \cup B_2$, we have
\begin{equation}\label{important_contextual_assum}
    \frac{|B_1|\cdot |B_2|}{(|B_1| + |B_2|)^2} \| \tilde X (\theta_{i,j} - \theta_{i,j+1}) \|^2 \geq C_\xi p \log T,
\end{equation}
where $C_\xi$ is some constant depending on $\xi$ only and it is the same constant in Assumption \ref{ass_context}, $C^\prime = C_\xi$.
Then there exists a constant $C$ dependent on the input $\xi$ such that if we run Algorithm \ref{alg:cp_ucb_context}, then we can detect all such changepoints with probability at least $1 - 1/T$.
\end{lem}
\begin{proof}[Proof Sketch] 
For ease of notation let $\theta_1 = \theta_{i,j}, \theta_2 = \theta_{i,j+1}$. Note that 
$(\tilde X_1 \hat \theta_1, \tilde X_2 \hat \theta_2)$ 
is the projection of $\tilde y$ onto the column space of the block diagonal matrix whose two blocks are $\tilde X_1,\tilde X_2$ respectively. We denote this projection as $P'$ and let $P$ be the projection onto the column space of $\tilde X$. Denote $\epsilon$ the 0-mean sub-Gaussian noise vector, we have $Z_{i,t,t'} = \|(P' - P) \tilde y\| \geq \|(P' - P) ((\tilde X_1 \theta_1)^\top, (\tilde X_2 \theta_2)^\top)^\top \| - \| (P' - P) \epsilon \|$, where we have decomposed $\tilde y$ into its mean and the noise term $\epsilon$. 
For the first term, by the properties of projection matrix and Assumption \ref{length_context}, one can verify that 
it is greater than $\Xi \frac{|B_1||B_2|}{(|B_1| + |B_2|)^2} \| \tilde X (\theta_1 - \theta_2) \|^2,$ where $\Xi = 1 - \frac{\xi (\xi - 1)}{2} > 0$.
For noise term, by idempotency, we know $\| (P' - P) \epsilon \|^2 = \epsilon^\top (P' - P) \epsilon$.
By Hanson-Wright inequality \cite{hsu2012tail}, we have
$\mathbbm P \left\{ \epsilon^\top (P' - P) \epsilon \ge p + 2 \sqrt{pu^{\prime}} + 2u^{\prime}  \right\} \le e^{-u^{\prime}}$
for $u^{\prime}>0$. So we detect the changepoint with probability at least $1-e^{-u^{\prime}}$ as long as 
$\frac{|B_1||B_2|}{(|B_1| + |B_2|)^2} \| \tilde X (\theta_1 - \theta_2) \|^2 \ge C p u^{\prime} > \Xi^{-1} (p + 2 \sqrt{pu^{\prime}} + 2u^{\prime}),
$ for some constant $C$ depending on $\Xi$ only.
We can set $u^{\prime} = 3 \log T + \log K$ and apply the union bound to obtain our desired result.
\end{proof}

Condition (2) in Lemma \ref{delay_context} holds when Assumption \ref{ass_context} holds with $C^\prime = C_\xi$, so we can bound detection delay in Lemma \ref{Lij_delay_context}.

\begin{lem}\label{Lij_delay_context}
Under Assumption \ref{sub_g}, \ref{length_context}, also Assumption \ref{ass_context} holds with $C^\prime = C_\xi$, the detection delay for changepoint $c_{i,j}$, where $i=1,\dots,K$ and $j=1,\dots, \gamma_i$, satisfies the following equation with probability at least $1-\frac{1}{T}$,
\begin{equation}
    L_{i,j} = \frac{4 C_{\xi} p \sqrt{T \log T }}{\delta^2_{i,j}}.
\end{equation}
\end{lem}

The following lemma bounds the false alarm probability. 

\begin{lem} \label{false_context}
Suppose that Assumption \ref{sub_g} holds.
There exists a constant $C$ dependent on the input $\xi$ such that if we run Algorithm \ref{alg:cp_ucb_context}, then the probability of any false alarm 
during a stationary region is bounded above by $1/T$.
\end{lem}
\begin{proof}[Proof Sketch] 
We can apply much of the same reasoning as in the previous proof.
Notice that within a stationary region, there is a single vector $\theta \in \mathbbm{R}^p$ such that $y = \tilde X \theta + \epsilon$, so $(P' - P)y = (P' - P)\epsilon$.
As before, we have that by the Hanson-Wright inequality,
$\mathbbm P \left\{ \epsilon^\top (P' - P) \epsilon \ge p + 2 \sqrt{pu^{\prime}} + 2u^{\prime}  \right\} \le e^{-u^{\prime}}.
$
So, we do not detect a changepoint for a single selection of $t^{\prime},t$ within a stationary region, with probability $1 - e^{-u^{\prime}}$ because $Cpu^{\prime} > p + 2 \sqrt{pu^{\prime}} + 2u^{\prime}$.
Furthermore, we can apply this result uniformly over the selections of arm, $t',t$ within stationary regions with the union bound,
and arrive at our conclusion by setting $u^{\prime} = 3 \log T + \log K$.
\end{proof}
Finally, we can piece all of these components together to obtain a regret bound that is nearly optimal \cite{wu2018learning,garivier2008upper}. 
\begin{thm}\label{regret_context}
Denote $\underline\delta = \min_{i=1}^K \min_{j=1}^{\gamma_i} \delta_{i,j}$ and $\Delta = \displaystyle{\max_{t=1}^T \max_{i \neq a_t}} \epsilon_{t,i}$. 
Under Assumption \ref{sub_g}, \ref{length_context}, and if Assumption \ref{ass_context} holds with $C^\prime = C_\xi$ depending on $\xi$ only, 
the regret bound for Multiscale-LinUCB satisfies the following:
\begin{equation}
\resizebox{.99\linewidth}{!}{$
    \displaystyle
    R(T) \leq \frac{4D C_{\xi} p\sqrt{T\log T}\Delta}{\underline\delta^2}  + 2 D R_{\text{LinUCB}}(S_{\max}) +  K \Delta \sqrt{T\log T}.
$}
\end{equation}%
\end{thm}


\section{Extensions}
\subsection{Non-stationary Joint Linear Models}
In addition to disjoint linear models, Chu et al. \shortcite{linucb} also proposed a contextual framework for joint linear model. We consider the extension of Multiscale-LinUCB for joint linear models below. This model is also consistent with the one considered by Russac et al. \shortcite{russac2019weighted} and Cheung et al. \shortcite{cheung2018learning}.
\begin{equation*}
    \mu_{t,i} = x_{t,i}^T \theta_{j},  c_{j-1} \le t < c_{j}. 
\end{equation*}
There are still $D$ changepoints $\mathcal{C} = \{c_0, \dots, c_{D-1}\}$ in total, where $c_0 = 1$ and $c_{D-1} = T+1$. However, the changepoints and model parameter is now invariant to the arms. In the $j$-th stationary period $[c_{j-1}, c_{j})$, each arm is associated with the same model parameter $\theta_j$. 

The analog of Multiscale-LinUCB algorithm for joint linear model is basically the same. 
However, now we only need to randomly preselect $\lceil \alpha T \rceil$ rounds in total, denoted as $\bar B$. 
For cut point $t^\prime \in (c_{i,\gamma}, t]$, we similarly define $B_1 = [c_{j}, t^{\prime}) \cap \bar B$ and $B_2 = [t^{\prime}, t] \cap \bar B$. $\tilde y_1 = (y_{t, I_t})_{t\in B_1}$, $\tilde y_2 = (y_{t, I_t})_{t\in B_2}$, $\tilde X_1 = (x_{t, I_t})_{t\in B_1}$ and $\tilde X_2 = (x_{t, I_t})_{t\in B_2}$. We also have $\tilde y = (\tilde y_1^T, \tilde y_2^T)^T$, $\tilde X = (\tilde X_1^T, \tilde X_2^T)^T$. We assert there is a changepoint if there exists a $t' \in (c_j, t]$ such that
\begin{equation*}
    Z_{t,t'}^2 := \| \tilde X_1(\hat\theta_1-\hat\theta) \|^2 + \| \tilde X_2 (\hat \theta_2 - \hat\theta) \|^2 \geq Cp\log T.
\end{equation*}
Here $\hat\theta_1, \hat\theta_2, \hat\theta$ is the same defined in Equation \ref{eq:theta_one}, \ref{eq:theta_two}, \ref{eq:theta}.
For joint linear model, we still have similar regret bounds.

\begin{thm}
Consider adjacent stationary periods $j,j+1$ of length $S_{j}$ and $S_{j+1}$ respectively, true parameter $\theta$ changes from $\theta_{j}$ to $\theta_{j+1}$. In these two stationary periods, there are two preselected sets $B_1$ and $B_2$ for changepoint detection only. Define $\delta_{j}^2 = \frac{1}{|B_1|+|B_2|}\sum_{t\in B_1 \cup B_2} \|x_{t,I_t}^T (\theta_{j} -\theta_{j+1})\|^2$. If there exists a constant $C_\xi$ depending only on $\xi$ in Assumption \ref{length_context} only, such that 
\begin{align*}
|B_1| \delta_{j}^2  &\geq 4C_\xi p \log T, \\
|B_2| \delta_{j}^2  &\geq 4C_\xi p \log T.
\end{align*}
Denote $\underline\delta^2 = \min_{j=1}^{D-1} \delta_{j}^2$ and $\Delta = \displaystyle{\max_{t=1}^T \max_{i \neq a_t}} \epsilon_{t,i}$. 
Under Assumption \ref{sub_g}, \ref{length_context}, 
the regret bound of Multiscale-LinUCB for joint linear models satisfies the following:
\begin{equation*}
\resizebox{.99\linewidth}{!}{$
    \displaystyle
    R(T) \leq   \frac{4D C_{\xi} p\sqrt{T\log T}\Delta}{\underline\delta^2}  + 2 D R_{\text{LinUCB}}(S_{\max})  + \Delta \sqrt{T\log T}.
$}
\end{equation*}%
\end{thm}

\subsection{Non-stationary Multi-armed Bandit (MAB)}
There are plenty of literature on non-stationary multi-armed bandit problems \cite{dis-sw,cusum,exp3r,exp,peter_new,mucb,besbes2014stochastic}. 
Most of the notations remain the same as in Section \ref{setup}. However, we don't have a model parameter $\theta$ now. For MAB setting, the algorithm can be simplified a lot. We don't need to randomly preselect rounds for changepoint detection purpose. Instead, at time $t$, we will now randomly select each arm with probability $\sqrt{\frac{\log T}{T}}$, and we will pull arm with the maximum UCB index with probability $1-K\sqrt{\frac{\log T}{T}}$. Define $B_1 = [ c_{i,\gamma}, t') \cap \{\text{time arm } i \text{ is pulled}\}$ and $B_2 = [t', t] \cap \{\text{time arm } i \text{ is pulled}\}.$ 
We calculate the test statistics $Z_{i,t,t^{\prime}}$ as follows.
\begin{equation}
 \label{test_stat}
    Z_{i,t,t^{\prime}} = \textstyle \frac{\sqrt{|B_1||B_2|}}{\sqrt{|B_1|+|B_2|}} \displaystyle  \left[ \frac{\sum_{s\in B_1} y_{s,i}}{|B_1|}  - \frac{\sum_{s\in B_2} y_{s,i}}{|B_2|} \right].
\end{equation}
If there exists a cut point $t' \in (c_{i,\gamma}, t]$ such that $Z_{i,t,t^{\prime}}^2 \geq 6\log T$, then we reset the most recent changepoint as current time $t$, we also reset the UCB index for arm $i$. Otherwise, we assert there is no changepoint in interval 
$[c_{i,\gamma}, t]$ and keep runing UCB.

Define $\delta_{i,j} = |\mu_{c_{i,j}-1,i} - \mu_{c_{i,j},i}|$ to be the change in reward of arm $i$ at 
changepoint $c_{i,j}$. Define $\delta = \min_{ \{(i,j): \delta_{i,j}\neq 0 \} } \delta_{i,j}$, $\underline \epsilon = \min_{t=1}^{T} \min_{i\neq a_t} \epsilon_{t,i},$ and $\bar \epsilon = \max_{t=1}^{T} \max_{i\neq a_t} \epsilon_{t,i}$. Without loss of generality, we can assume $\delta_{i,j} \leq M$ for some constant $M>0$ for all $i,j$.
We provide an analog of Assumption \ref{ass_context} and 
regret analysis in MAB setting.

\begin{assum}\label{gap}
\textbf{(Detectability)}
$S_{i,j}$ and $S_{i,j+1}$ are the length of two adjacent stationary periods for arm $i$,
for all $j \in \{1,\dots, {\gamma_i-1} \},$  we assume there exists a constant $C^{\prime} \geq \max(24\times (1+\frac{1}{\sqrt{3}})^2, 8M^2)$, such that,
\begin{equation*}
  \delta_{i,j}^2  \min (S_{i,j}, S_{i,j+1})  \geq C^{\prime} \sqrt{T\log T}.
\end{equation*}
\end{assum}

\begin{figure*}[tb]
  \begin{center}
    \includegraphics[width=0.5\columnwidth]{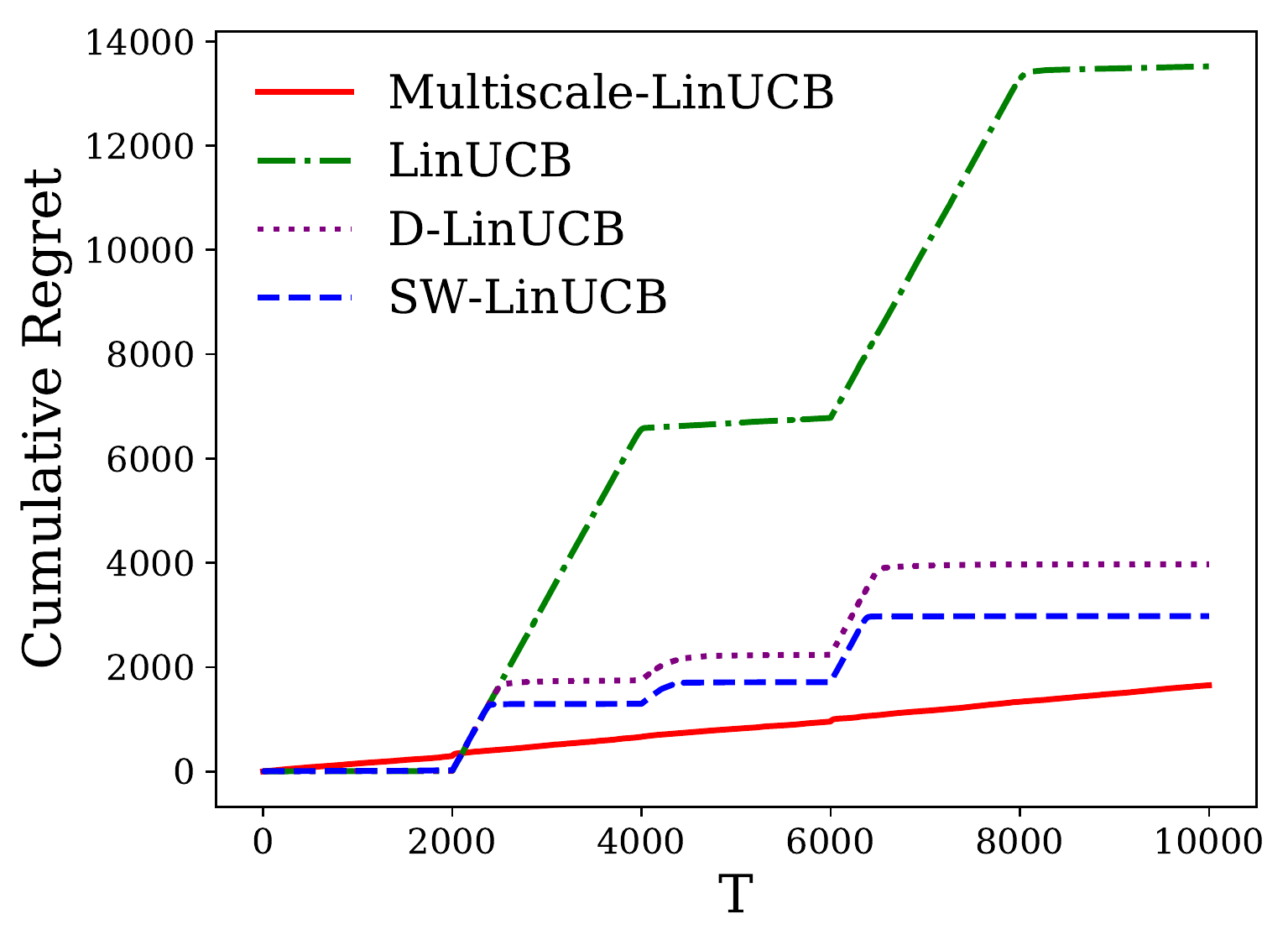} \includegraphics[width=0.5\columnwidth]{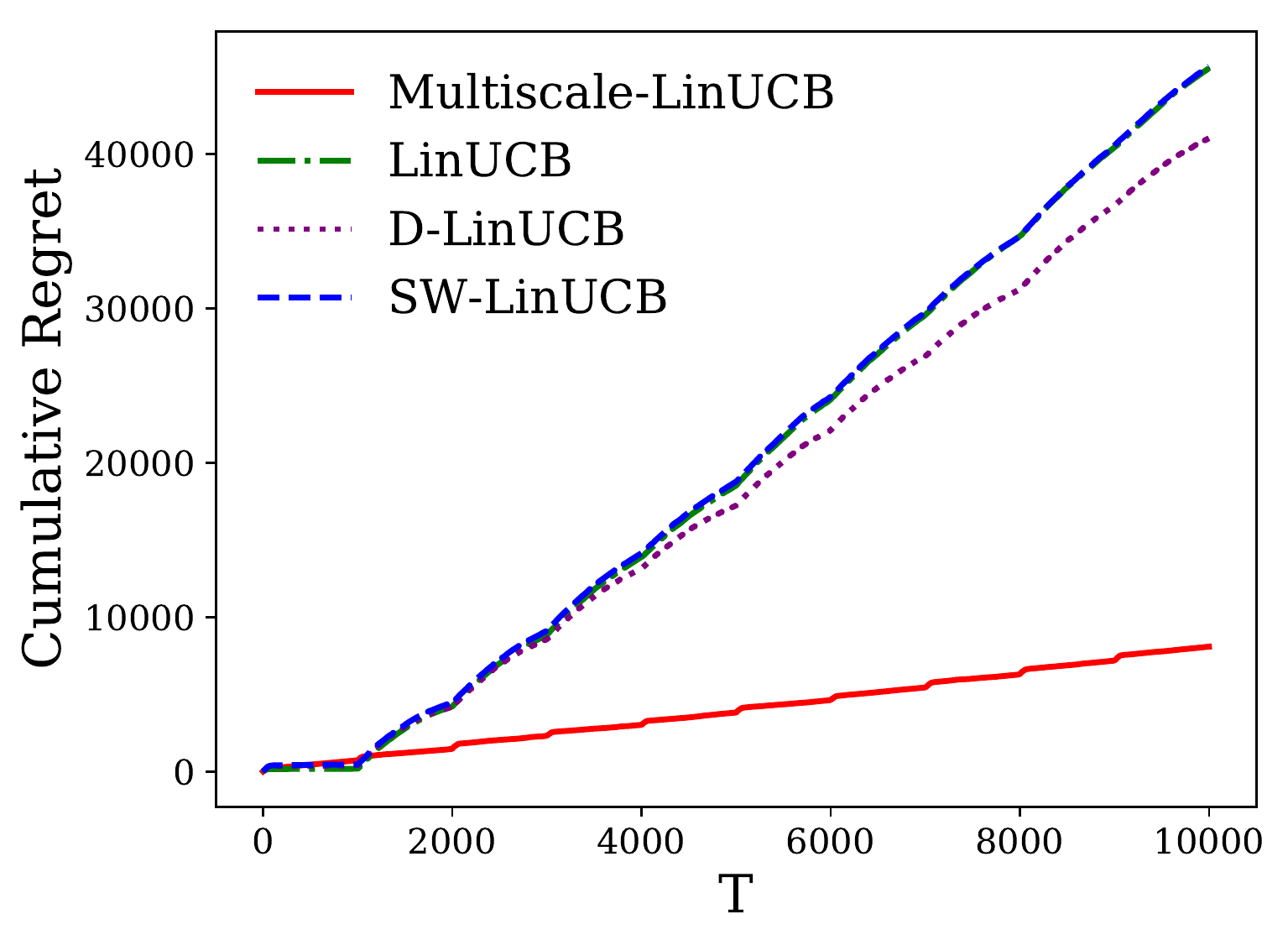}
    \includegraphics[width=0.5\columnwidth]{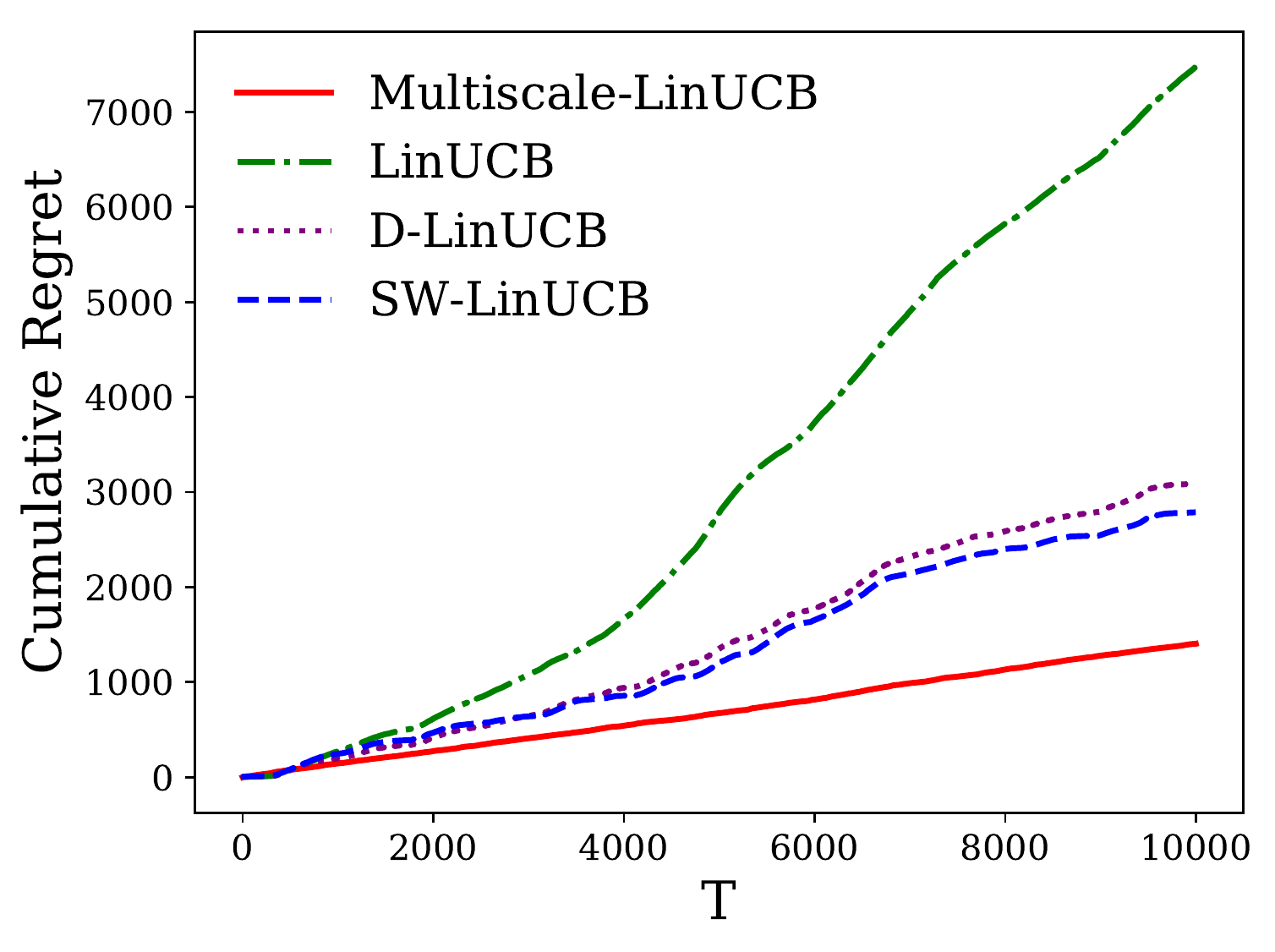}
    \includegraphics[width=0.5\columnwidth]{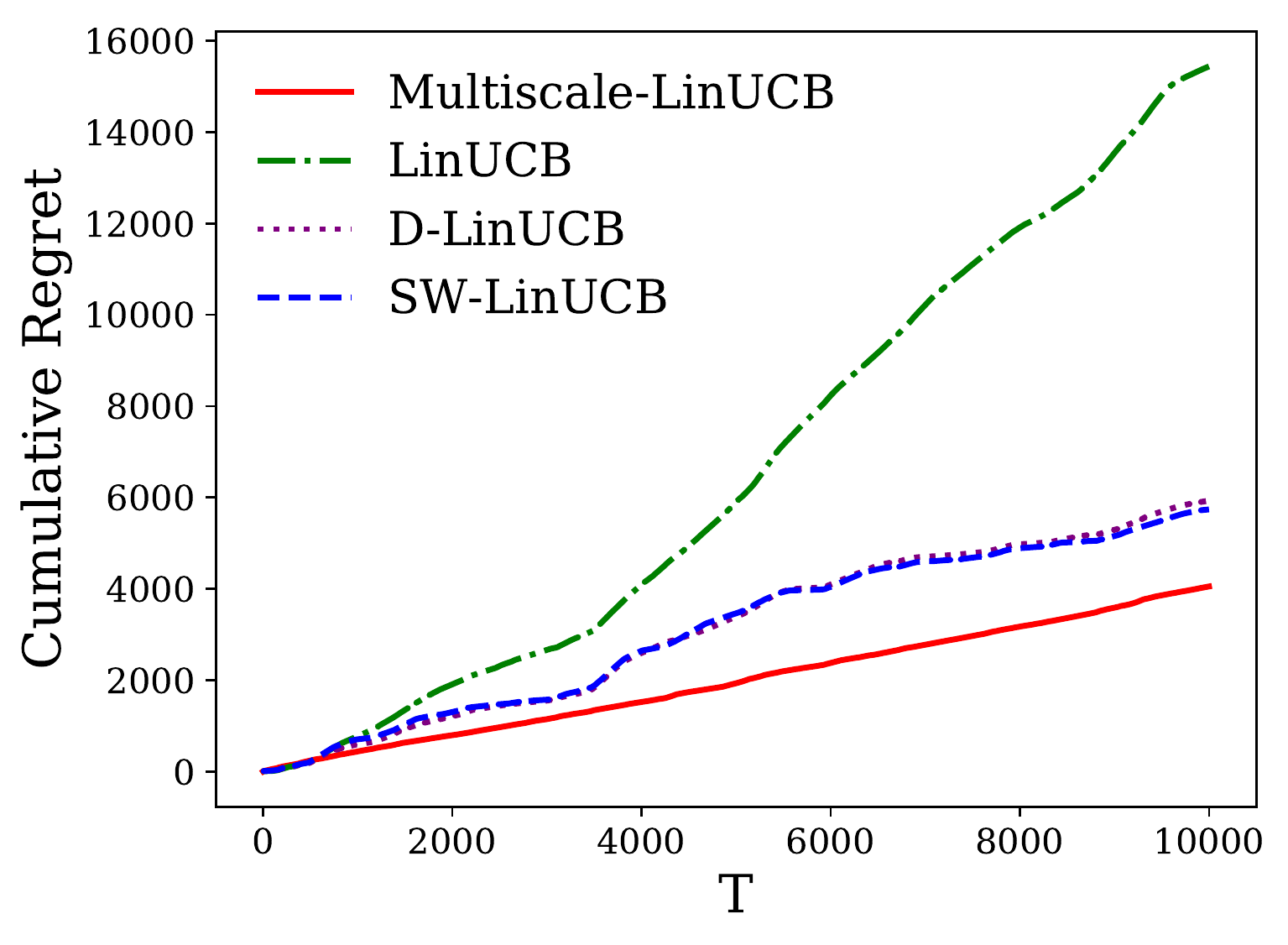}
  \end{center}
  \vskip -0.2in
  \caption{Starting from left to right are plots for Scenario 1, 2, 3, 4 respectively (SW-LinUCB and LinUCB are almost overlapped in the second plot).}
  \label{main_res}
\end{figure*}

\begin{thm}\label{regret}
Under Assumption \ref{sub_g}, \ref{gap}, the regret of our proposed algorithm in non-stationary MAB setting satisfies the following with probability at least $1-\frac{4}{T}$, 
\begin{align*}\label{reg_bound_mab}
    R(T) 
    &\leq D [2 R_{\text{UCB}}(S_{\max}) + \frac{2C^{\prime} \sqrt{T\log T} }{\max\left(\underline\epsilon, \delta \right)}] + \bar \epsilon  K \sqrt{T\log T} \nonumber \\
    &= O\left( \frac{D \sqrt{T\log T}}{\max\left(\underline\epsilon, \delta \right)} + \bar\epsilon K\sqrt{T\log T}\right).
\end{align*}
\end{thm}

\begin{remark}
For our proposed algorithm in non-stationary MAB setting, we don't need any input of unknown information like total number of changepoints, which is a big advantage over many existing non-stationary MAB algorithms \cite{cusum,dis-sw,besbes2014stochastic}. Moreover, our algorithm can achieve (nearly) optimal regret bound \cite{dis-sw}.
\end{remark}

\section{Experimental Results}\label{exps}
In Algorithm \ref{alg:cp}, although the algorithm breaks at time $t^{\prime}$, which is a cut point of interval $[c_{i,\gamma}, t]$, the returned changepoint is the current time $t$. However, we found that reusing information 
during $[t^{\prime}, t]$ is helpful for reducing cumulative regret. Therefore, in the experiments below, we use $t^{\prime}$ as the detected changepoints instead of $t$. 
We compare our algorithm with state-of-the-art algorithms including Sliding Window LinUCB (SW-LinUCB) \cite{cheung2018learning}, D-LinUCB \cite{russac2019weighted} and LinUCB \cite{linucb}. We omit the comparison with Dynamic Linear UCB (dLinUCB) here since that it is shown by Russac et al. \shortcite{russac2019weighted} in their experiments that dLinUCB performs much worse than D-LinUCB, and even worse than LinUCB in many simulations, which is also the case in our experiments.

For Multiscale-LinUCB, although input $C$ needs to be chosen based on $\xi$ in order to achieve our regret bound in the analysis, we found that in most experiments, choosing $C = \frac{1}{\log T} (1+2\sqrt{\frac{3\log T + \log K}{p}} + \frac{6\log T + 2\log K}{p})$ is enough. For both SW-LinUCB \cite{cheung2018learning} and D-LinUCB \cite{russac2019weighted}, the algorithm needs to know $B_T$ which is an upper bound on $\sum_{t=1}^{T-1} \|\theta_{t+1} - \theta_t\|$. Here $\theta_t$ is the true model parameter at time $t$. However, in practice, it's often the case that $B_T$ is unknown. 
The authors of SW-LinUCB \cite{cheung2018learning} suggest using $B_T=1$ when $B_T$ is unknown, so we use $B_T=1$ in the comparisons.

All the experiments shown here is for non-stationary contextual bandit with joint linear model, since SW-LinUCB and D-LinUCB is proposed for joint linear models. For all experiments, we fix $T=10,000$ and draw sample reward from $N(\mu_{t,i}, 1)$, where $\mu_{t,i}$ is the mean reward for arm $i$ at time $t$. Feature vectors $x_{t,i}$ are drawn from $U(0,10)$ randomly. In each stationary period, the true model parameter $\theta$ is drawn from $U(-1,1)$, except the first scenario.
We repeat the experiments $10$ times and plot the average regret of these experiments. We demonstrate the success of Multiscale-LinUCB under the scenarios below. Note that for both Scenario 1 and 2, if you zoom in and look closely at the plots, you will find that the regret of Multiscale-LinUCB accumulates at a faster rate at changepoints. Immediately after the changepoints, the regret accumulates much slower which shows that our algorithm captures the change very quickly and adapts well to the changing environment. Details can be found in Figure \ref{main_res}.

\begin{enumerate}
    \item \textbf{Scenario 1 (Detectable environments)}: This is a similar setting as in \cite{russac2019weighted}. Before $t=2000$, $\theta = (1,0)$; for $t\in [2000, 4000]$, $\theta = (-1,0)$; for $t\in [4000, 6000]$, $\theta = (0,1)$; for $t>6000$, $\theta = (0,-1)$. We set $K=2$ and $p=2$ here. From the plot in Figure \ref{main_res}, we can see that LinUCB cannot adpat to abruptly changing environment. SW-LinUCB and D-LinUCB presents similar behavior when there is an abrupt changepoint. Both algorithms incur fairly large regret for some rounds right after the changepoint. However, Multiscale-LinUCB can adapt to the change faster and therefore achieve smaller regret.
    \item \textbf{Scenario 2 (High dimensions)}: There are in total $10$ changepoints and they are evenly spread over the whole time horizon. We set $K=2$ and $p=50$ here. In the experiments of \cite{russac2019weighted}, it was shown that under high dimensions ($p=50$), D-LinUCB can perform well. We can see from Figure \ref{main_res} that Multiscale-LinUCB adapts to changes much faster and performs much better than all other algorithms under high dimensions.
    \item \textbf{Scenario 3 (Random changepoints)}: At time $t$, $\theta$ changes with probability $\frac{10}{T}$. We set $K=2$ and $p=2$. Although we require each stationary period to be long enough in Assumption \ref{ass_context} and \ref{length_context}, we show here that even when the changepoints are randomly distributed over the whole time horizon, where Assumption \ref{ass_context}, \ref{length_context} could be violated, Multiscale-LinUCB still performs quite well.
    \item \textbf{Scenario 4 (Multiple arms)}: At each time $t$, $\theta$ will change with probability $\frac{10}{T}$. We set $K=4$ and $p=2$ here. We show by this scenario that Multiscale-LinUCB can work well with multiple arms. We found that the regret of every algorithm roughly scales linearly with the number arms, although the regret analysis for D-LinUCB and SW-LinUCB shows that their regret upper bound is invariant to $K$ \cite{russac2019weighted}.
\end{enumerate}

\section{Conclusion}
We proposed a multiscale changepoint detection based LinUCB algorithm for non-stationary stochastic disjoint linear bandit setting, Multiscale-LinUCB. We also extended it to non-stationary joint linear bandit setting and MAB setting. 
The regret of our proposed algorithm matches the lower bound up to a logarithm factor. Particularly, our algorithm can also deal with faint change in mean reward. Experimental results show that our proposed algorithm outperforms other state-of-the-art algorithms significantly in non-stationary environments.

\newpage

\bibliographystyle{named}
\bibliography{ijcai20}

\newpage

\appendix

\section{Proofs for Non-stationary Contextual Bandit Setting}

\subsection{Proof of Lemma \ref{delay_context}}
\begin{proof}
For ease of notation let $\theta_1 = \theta_{i,j}, \theta_2 = \theta_{i,j+1}$.
Consider the test statistic defined in \eqref{eq:context_stat},
\[
Z_{i,t,t'}^2 = \left\| \begin{array}{c} \tilde X_1 (\hat \theta_1 - \hat \theta) \\ \tilde X_1 (\hat \theta_2 - \hat \theta) \end{array} \right\|^2.
\]
By standard OLS theory, the vector $(\tilde X_1 \hat \theta_1, \tilde X_2 \hat \theta_2)$ is the projection of $\tilde y$ onto the column space of the following matrix,
\[
\left( \begin{array}{cc}
    \tilde X_1 & 0 \\
    0 & \tilde X_2
\end{array} \right).
\]
Let's call this projection $P'$, and let $P$ be the projection onto the column space of $\tilde X$.
Then 
\[
\left\| \begin{array}{c} \tilde X_1 (\hat \theta_1 - \hat \theta) \\ \tilde X_1 (\hat \theta_2 - \hat \theta) \end{array} \right\|^2 = \|(P' - P) \tilde y\|^2.
\]
Notice that these column spaces are nested, so that $P' - P$ is the projection onto a subspace orthogonal to the column space of $\tilde X$.
Let $\tilde y^\top = ((\tilde X_1 \theta_1)^\top, (\tilde X_2 \theta_2)^\top) + \epsilon^\top$ for zero-mean subGaussian(1) vector, $\epsilon$.
By the triangle inequality, we have that
\[
\|(P' - P) \tilde y\| \ge \left\| (P' - P) \left( \begin{array}{c}
     \tilde X_1 \theta_1  \\
     \tilde X_2 \theta_2 
\end{array} \right)
\right\| - \| (P' - P) \epsilon \|.
\]
Let us begin by lower bounding the first term on the RHS.
Notice that for any vector $\theta' \in \mathbbm R^p$ we have that 
\[
\left\| (P' - P) \left( \begin{array}{c}
     \tilde X_1 \theta_1  \\
     \tilde X_2 \theta_2 
\end{array} \right)
\right\|
= \left\| (P' - P) \left( \begin{array}{c}
     \tilde X_1 (\theta_1 - \theta')  \\
     \tilde X_2 (\theta_2 - \theta')
\end{array} \right)
\right\|,
\]
since $P \tilde X \theta' = P' \tilde X \theta'$.
Let 
\[
\theta' = \zeta \theta_1 + (1 - \zeta) \theta_2,
\]
where $\zeta = |B_1| / (|B_1| + |B_2|)$ and denote $\delta = \theta_1 - \theta_2$.
Hence,
\[
\left\| (P' - P) \left( \begin{array}{c}
     \tilde X_1 \theta_1  \\
     \tilde X_2 \theta_2 
\end{array} \right)
\right\| 
=
\left\| (P' - P) \left( \begin{array}{c}
     (1 - \zeta) \tilde X_1 \delta  \\
     - \zeta \tilde X_2 \delta 
\end{array} \right)
\right\|
\]
Because, the projections are into nested subspaces, and the vector in question is within the outer subspace, we have that
\begin{eqnarray*}
& & \left\| (P' - P) \left( \begin{array}{c}
     (1 - \zeta) \tilde X_1 \delta  \\
     - \zeta \tilde X_2 \delta 
\end{array} \right)
\right\|^2 \\
&=&
\left\| \left( \begin{array}{c}
     (1 - \zeta) \tilde X_1 \delta  \\
     - \zeta \tilde X_2 \delta 
\end{array} \right)
\right\|^2 
-
\left\| P \left( \begin{array}{c}
     (1 - \zeta) \tilde X_1 \delta  \\
     - \zeta \tilde X_2 \delta 
\end{array} \right)
\right\|^2.
\end{eqnarray*}
The first term can be written as,
\begin{eqnarray*}
& & \left\| \left( \begin{array}{c}
     (1 - \zeta) \tilde X_1 \delta  \\
     - \zeta \tilde X_2 \delta 
\end{array} \right)
\right\|^2 \\
&=&  (1 - \zeta)^2 \delta^\top \tilde X_1^\top \tilde X_1 \delta + \zeta^2 \delta^\top \tilde X_2^\top \tilde X_2 \delta \\
&=& \frac{|B_1| \cdot |B_2|}{(|B_1| + |B_2|)^2} \delta^\top \tilde X^\top \tilde X \delta
\end{eqnarray*}
The second term can be written as,
\begin{eqnarray*}
& & \left\| P \left( \begin{array}{c}
     (1 - \zeta) \tilde X_1 \delta  \\
     - \zeta \tilde X_2 \delta 
\end{array} \right)
\right\|   \\
&=&
\| \tilde X (\tilde X^\top \tilde X)^{-1} \left( (1 - \zeta)\tilde X_1^\top \tilde X_1 \delta - \zeta \tilde X_2^\top \tilde X_2 \delta \right)\|.
\end{eqnarray*}
Notice that 
\begin{equation*}
(1 - \zeta)\tilde X_1^\top \tilde X_1 \delta - \zeta \tilde X_2^\top \tilde X_2 \delta = \frac{|B_1||B_2|}{|B_1| + |B_2|} (\tilde \Sigma_1 - \tilde \Sigma_2) \delta.
\end{equation*}
Thus,
\begin{eqnarray*}
& & \left\| P \left( \begin{array}{c}
     (1 - \zeta) \tilde X_1 \delta  \\
     - \zeta \tilde X_2 \delta 
\end{array} \right)
\right\|^2 \\
&=&
\frac{(|B_1||B_2|)^2}{(|B_1| + |B_2|)^2} \delta^\top (\tilde \Sigma_1 - \tilde \Sigma_2) (\tilde X^\top \tilde X)^{-1} (\tilde \Sigma_1 - \tilde \Sigma_2) \delta.
\end{eqnarray*}
Notice that $\sqrt{|B_1||B_2|}/(|B_1| + |B_2|) = \sqrt{\zeta (1- \zeta)}$ and
\begin{eqnarray*}
\sqrt{\zeta (1- \zeta)} (\tilde \Sigma_1 - \tilde \Sigma_2 ) \preceq (\xi - 1) \sqrt{\zeta (1- \zeta)}  \tilde \Sigma_2.
\end{eqnarray*}
Moreover, $\sqrt{\zeta (1- \zeta)} \le 1/2$ and
\begin{eqnarray*}
& & (\xi - 1) \sqrt{\zeta (1- \zeta)}  \tilde \Sigma_2 \\
&\preceq & \frac{(\xi - 1)}{2} (\zeta \tilde \Sigma_2 + (1-\zeta) \tilde \Sigma_2 ) \\
&\preceq & \frac{\xi (\xi - 1)}{2} (\zeta \tilde \Sigma_1 + (1-\zeta) \tilde \Sigma_2 ) = \frac{\xi (\xi - 1)}{2} \frac{\tilde X^\top \tilde X}{|B_1| + |B_2|}.
\end{eqnarray*}
The same holds if we switch the roles of $\tilde \Sigma_1, \tilde \Sigma_2$.
Hence,
\begin{equation*}
\left\| P \left( \begin{array}{c}
     (1 - \zeta) \tilde X_1 \delta  \\
     - \zeta \tilde X_2 \delta 
\end{array} \right)
\right\|^2 
\le 
\frac{|B_1||B_2|}{(|B_1| + |B_2|)^2} \frac{\xi (\xi - 1)}{2} \delta^\top \tilde X^\top \tilde X \delta.
\end{equation*}
Hence,
\[
\left\| (P' - P) \left( \begin{array}{c}
     \tilde X_1 \theta_1  \\
     \tilde X_2 \theta_2 
\end{array} \right)
\right\| \ge \Xi \frac{|B_1||B_2|}{(|B_1| + |B_2|)^2} \| \tilde X \delta \|^2,
\]
where $\Xi = 1 - \frac{\xi (\xi - 1)}{2} > 0$.
We will now control the noise term, $\| (P' - P) \epsilon \|^2$.
Due to idempotency,
\[
\| (P' - P) \epsilon \|^2 = \epsilon^\top (P' - P) \epsilon.
\]
By the Hanson-Wright inequality (the form in \cite{hsu2012tail} is sufficient for our purposes),
\[
\mathbbm P \left\{ \epsilon^\top (P' - P) \epsilon \ge p + 2 \sqrt{pu^{\prime}} + 2u^{\prime}  \right\} \le e^{-u^{\prime}},
\]
for $u^{\prime}>0$.
To see this note that ${\rm tr}(P'-P) = p$ (by Assumption \ref{length_context}) and $\| P' - P \| = 1$.
So we detect the changepoint with probability at least $1-e^{-u^{\prime}}$ as long as 
\[
\frac{|B_1||B_2|}{(|B_1| + |B_2|)^2} \| \tilde X \delta \|^2 \ge C p u^{\prime} > \Xi^{-1} (p + 2 \sqrt{pu^{\prime}} + 2u^{\prime}),
\]
for $u^{\prime} > 1$ and some constant $C$ depending on $\Xi$ only.
We can set $u^{\prime} = 3 \log T + \log K$ and apply the union bound to obtain our desired result.
\end{proof}

\subsection{Proof of Lemma \ref{Lij_delay_context}}
Before we prove Lemma \ref{Lij_delay_context}, we first state our Lemma \ref{length} below. Then we will use it to show that our detection delay is bounded from above.

We are uniformly sampling each arm with a small probability $\alpha$ at each time, therefore, when the stationary period is long enough, we can get sufficient samples to estimate the expected reward of every arm. This is made precise in the following lemma.

\begin{lem}\label{length}
For any stationary period with length $S$, for any arm $i$, we have
\begin{equation*}
    P\left(\text{\# of times arm i is pulled} < \frac{1}{2} S\alpha \right) \le e^{- \frac{\alpha S}{8}}.
\end{equation*}
\end{lem}

\begin{proof}
For any time $t$ within this stationary period, arm $i$ is pulled with probability $p_t \geq \alpha = \sqrt{\frac{\log T}{T}}$. Define $N_t \sim Ber(p_t)$, then $\sum_{t=1}^S N_t$ is the number of times arm $i$ is pulled in this stationary period of length S. Define $n= E[\sum_{t=1}^{S} N_t] = \sum_{t=1}^{S} p_t \geq S\alpha.$ 
\begin{eqnarray*}
& & P\left(\sum_{t=1}^{S} N_t \geq \frac{1}{2} S \alpha\right) \\
&=& 1- P\left(\sum_{t=1}^{S} N_t - n < \frac{1}{2} S \alpha - n \right)\\
& \geq & 1- P\left(\sum_{t=1}^{S} N_t - n < -\frac{1}{2} n  \right) \\
&\geq & 1- exp\left(- \frac{1}{8} n\right) \geq 1- exp\left(- \frac{1}{8} S \alpha\right).
\end{eqnarray*}
The last step is obtained from Chernoff bound.
\end{proof}

Now we can formally prove our Lemma \ref{Lij_delay_context}.

\begin{proof}
We firstly show that Assumption \ref{ass_context} implies Equation \ref{important_contextual_assum} in Lemma \ref{delay_context}.
From Lemma \ref{length}, we have the following:
\begin{align*}
    P(|B_1| \geq \frac{1}{2} S_{i,j} \alpha) \geq 1- exp(-\frac{1}{8} S_{i,j} \alpha) ,\\
    P(|B_2| \geq \frac{1}{2} S_{i,j+1} \alpha) \geq 1- exp(-\frac{1}{8} S_{i,j+1} \alpha).
\end{align*}
Note that from Assumption \ref{ass_context}, we have
\begin{equation*}
    S_{i,j} \alpha, S_{i,j+1} \alpha \geq \frac{4p C_{\xi} \log T}{\delta^2_{i,j}}.
\end{equation*}
Therefore, $|B_1|, |B_2| \geq \frac{2 C_{\xi} p \log T}{\delta^2_{i,j}}$ with probability at least $1- exp(-\frac{p C_{\xi} \log T}{2\delta^2_{i,j}})$. So we get with probability $1- exp(-\frac{p C_{\xi} \log T}{2\delta^2_{i,j}})$,
\begin{eqnarray*}
  & & \frac{|B_1|\cdot |B_2|}{(|B_1| + |B_2|)^2} \| \tilde X (\theta_{i,j} - \theta_{i,j+1}) \|^2 \\
  &\geq & \frac{|B_1|\cdot |B_2|}{(|B_1| + |B_2|)} \delta_{i,j}^2 \geq  C_\xi p \log T .
\end{eqnarray*}
 
Now, assume changepoint $c_{i,j-1}$ has already been detected and we are trying to detect changepoint $c_{i,j}$.
From Lemma \ref{delay_context}, we know that if we are at time $c_{i,j} + L_{i,j}$, and $\tilde B_2 = [c_{i,j}, c_{i,j} + L_{i,j}] \cap \bar B_i$ is the pre-sampled time slots for arm $i$, then we will be able to detect changepoint $c_{i,j}$ when $|\tilde B_2| \geq \frac{2 C_{\xi} p \log T}{\delta^2_{i,j}}$. Notice that $L_{i,j}$ here is exactly the detection delay for changepoint $c_{i,j}$.

Again, using Lemma \ref{length}, we know when $L_{i,j} = \frac{4 C_{\xi} p \sqrt{T \log T }}{\delta^2_{i,j}}$,
\begin{eqnarray*}
& & P( |\tilde B_2| \geq \frac{2 C_{\xi} p \log T}{\delta^2_{i,j}} ) =  P(|\tilde B_2| \geq \frac{1}{2} L_{i,j} \alpha ) \\
&\geq & 1- exp(-\frac{1}{8} L_{i,j} \alpha ) =  1- exp(- \frac{ C_{\xi} p  \log T }{2\delta^2_{i,j}} ).
\end{eqnarray*}

When $C_{\xi}p \geq 2\delta^2_{i,j}$ for every $i=1,\dots,K$ and every $j=1,\dots,\gamma_i$, 
then detection delay is $L_{i,j} = \frac{4 C_{\xi} p \sqrt{T \log T }}{\delta^2_{i,j}}$ with probability at least $1-\frac{1}{T}$.
\end{proof}

\subsection{Proof of Theorem \ref{regret_context}}
\begin{proof}
If the changepoints can all be perfectly detected by Algorithm~\ref{alg:cp}, then the total regret $R(T) \leq \sum_{j=1}^{D} R_{\text{LinUCB}}(S_j) + \alpha TK\Delta$, where $R_{\text{UCB}}(S_j)$ is the UCB-regret in the $j$-th stationary period ($S_j$ is the length of this stationary period) and the second term is caused by uniform sampling of the arm $i$ with probability $\alpha$ every time. But if changepoints can't be perfectly detected, there may be false alarms and detection delays. So we have 
\begin{equation*}
    R(T) \leq \sum_{j=1}^D R_{\text{LinUCB}}(S_j) + \sum_{i=1}^K \alpha T \Delta + R_{\text{delay}}(T) + R_{\text{false}}(T),
\end{equation*}
where $S_j$ is the length of the $j$-th stationary period, $R_{\text{false}}(T)$ is the regret caused by false alarms up to time $T$, and $R_{\text{delay}}(T)$ is the regret caused by detection delays up to time $T$.
Firstly, it was shown in \cite{news_contextual} that $R_{\text{LinUCB}}(T) \sim \tilde O(\sqrt{p K T})$.
Secondly, $\sum_{i=1}^K \alpha T \Delta = K \Delta \sqrt{T\log T}$.
Thirdly, from Lemma \ref{false_context}, we have $R_{\text{false}}(T) \leq \sum^{D}_{j=1} R_{\text{LinUCB}}(S_j) \sim \tilde O(\sqrt{p K T})$.
Finally, from Lemma \ref{Lij_delay_context}, we have 
$R_{\text{delay}}(T) \leq \sum^K_{i=1} \sum_{j=1}^{\gamma_i} \Delta L_{i,j} \leq  \frac{4D C_{\xi} p\sqrt{T\log T}\Delta}{\underline\delta^2}$.
Combining the above, we get the conclusion.
\end{proof}

\section{Proof of Proposition \ref{prop:assum4}}\label{assum4}

\begin{proof}
Assume the second moment matrix is $V$.
Since $V$ is positive definite, so
\begin{eqnarray}\label{conclusion}
\xi \tilde \Sigma_2 \succeq \tilde \Sigma_1 \succeq \xi^{-1} \tilde \Sigma_2
\end{eqnarray}
is equivalent to 
\begin{eqnarray*}
  \xi V^{-\frac{1}{2}} \tilde \Sigma_2 V^{-\frac{1}{2}} \succeq  V^{-\frac{1}{2}} \tilde \Sigma_1 V^{-\frac{1}{2}} \succeq \xi^{-1} V^{-\frac{1}{2}} \tilde \Sigma_2 V^{-\frac{1}{2}}. 
\end{eqnarray*}

Notice that $V^{-\frac{1}{2}} \tilde \Sigma_j V^{-\frac{1}{2}} = \frac{1}{|B_j|}  ({ \tilde X_j V^{-\frac{1}{2}} })^T \tilde X_j V^{-\frac{1}{2}}$ and
$\tilde X_j V^{-\frac{1}{2}}$ is a matrix with sub-Gaussian isotropic rows, therefore, from \cite{vershynin2010introduction}, we have that the following with probability at least $1-\frac{2}{T}$
for $j=1,2$,
\begin{eqnarray*}
    & & (1 - q_j)^2 \leq 
    \lambda_{\min} (V^{-\frac{1}{2}} \tilde \Sigma_j V^{-\frac{1}{2}} ) \nonumber\\
    & \leq &
    \lambda_{\max} (V^{-\frac{1}{2}} \tilde \Sigma_j V^{-\frac{1}{2}} ) 
    \leq (1+q_j)^2
\end{eqnarray*}
where $q_j = c_2 \sqrt{\frac{p}{|B_j|}} + \sqrt{\frac{\log T}{c_1 |B_j|}}$. Here $c_1,c_2$ are two positive constants depending on the sub-Gaussian norm of row vectors only. $\lambda_{\min} (\cdot)$ and $\lambda_{\max} (\cdot)$ represents the minimum and maximum eigenvalue of a matrix respectively.

When the consecutive stationary periods are of length $S_1, S_2$ respectively and $S_1, S_2 =\omega(\sqrt{T\log T})$, we have 
$|B_1|, |B_2| \gg \log T$ with probability at least $1-\frac{1}{T}$. Assumption \ref{length_context} requires $1<\xi < 2$. 
By applying a union bound, we have with probability at least $1-\frac{3}{T}$,
we have
\begin{equation*}
  \frac{ \lambda_{\max} (V^{-\frac{1}{2}} \tilde \Sigma_1 V^{-\frac{1}{2}} ) }{ \lambda_{\min} (V^{-\frac{1}{2}} \tilde \Sigma_2 V^{-\frac{1}{2}} ) } \leq 
  \frac{(1+q_1)^2}{(1-q_2)^2} \to 1 < 2.
\end{equation*}
Similarly, we have with probability at least $1-\frac{3}{T}$,
\begin{equation*}
  \frac{ \lambda_{\max} (V^{-\frac{1}{2}} \tilde \Sigma_2 V^{-\frac{1}{2}} ) }{ \lambda_{\min} (V^{-\frac{1}{2}} \tilde \Sigma_1 V^{-\frac{1}{2}} ) } \leq 
  \frac{(1+q_2)^2}{(1-q_1)^2} \to 1 < 2.
\end{equation*}
So we get Assumption \ref{length_context} holds with probability greater than $1-\frac{3}{T}$.
\end{proof}

\section{Non-stationary Multi-armed Bandit Setting}
We present the extension of our algorithm to non-stationary MAB setting below.
\begin{align}
    C_t(i) &= \sum_{s=c_{i,\gamma}}^t \mathbbm{1}_{\{I_s = i\}},  \label{ucb_eq1}\\
    \bar Y_t(i) &= \frac{1}{C_t(i)} \sum_{s=c_{i,\gamma}}^t Y_{s,i} \mathbbm{1}_{\{I_s = i\}}. \label{ucb_eq2} \\
    I_t &= \left\{ \begin{array}{ll}
    \argmax \bar Y_t(i) + \sqrt{\frac{2 \log t}{C_t(i)}}, &\text{w.p. } 1-K \alpha \\
    \text{uniform draw from } \mathcal{K}, &\text{w.p. } K \alpha.
    \end{array}  \right.  \label{ucb_eq3}
\end{align} 

The analysis of regret bound for non-stationary MAB setting follows similar ways as non-stationary contextual setting. 
In fact the statistic \eqref{eq:context_stat} reduces to its MAB counterpart when the design matrix is purely intercept.
\begin{algorithm}[H]
   \caption{Multiscale-UCB algorithm for MAB} 
   \label{alg:cp_ucb}
\begin{algorithmic}[1]
   \STATE {\bfseries Input:} $\alpha, T$
   \STATE $c_{i,\gamma} \gets 1$ for all $i\in \{1,\dots,K\}$
   \STATE Play each arm once, observe reward $\{Y_{1,1},\dots,Y_{K,K}\}$
   \FOR{$t = K+1$ {\bfseries to} $T$}
        \STATE Play arm $I_t$, obtained by Equation \eqref{ucb_eq1}, \eqref{ucb_eq2}, \eqref{ucb_eq3}
        \STATE Observe $Y_{t,I_t}$
        \STATE Run Algorithm \ref{alg:cp} with input $I_t, t, c_{I_t,\gamma}$, where $Z_{i,t,t^\prime}$ is defined in Equation \ref{test_stat}
        \IF{there is a changepoint}
            \STATE $c_{I_t,\gamma} \gets t$
            \STATE Reset \eqref{ucb_eq1}, \eqref{ucb_eq2} for arm $I_t$.
        \ENDIF
   \ENDFOR
\end{algorithmic}
\end{algorithm}

\begin{proposition}
\label{prop:reduces}
The test statistic for the contextual setting, defined in \eqref{eq:context_stat}, reduces to \eqref{test_stat} in the MAB setting ($x_{t,i} = 1$). 
\end{proposition}
\begin{proof}
Notice that when $\tilde X_1 = 1_{B_1}, \tilde X_2 = 1_{B_2}$ then the projection $P' - P$ is the projection onto the subspace within span of $1_{B_1}, 1_{B_2}$ and orthogonal to $1_{B_1 \cup B_2}$.  This space has dimension 1 and is the span of the Haar wavelet vector
\[
\psi = \frac{\sqrt{|B_1||B_2|}}{\sqrt{|B_1| + |B_2|}} \left( \frac{1_{B_1}}{|B_1|} -  \frac{1_{B_2}}{|B_2|} \right).
\]
Notice that $\|\psi\| = 1$.  So,
\[
(P' - P) = \psi^\top \tilde y,
\]
which is precisely \eqref{test_stat}.
\end{proof}

\subsection{False Alarm}
\begin{lem}\label{prob_false_alarm}
For a time $t$ during any stationary period $[c_{i,j-1}, c_{i,j})$, if Algorithm \ref{alg:cp} is run with $u = \sqrt{6\log T}$, then the probability of false alarm at time $t$ is $P_f(t) < \frac{1}{T^2}$.
\end{lem}

\begin{proof}
False alarm means that the algorithm detected a false changepoint during a stationary period. Consider a time $t$ in a stationary period for arm $i$, $[c_{i,j}, c_{i,j+1})$, we bound the probability of false alarm for arm $i$ at time $t$ below. For a $t^{\prime} \in (c_{i,j}, t]$, define $B_1 = \{ c_{i,j}, \dots, t^{\prime}-1\} \cap \{\text{time arm } i \text{ is pulled}\}$ and $B_2 = \{t^{\prime},\dots, t\} \cap \{\text{time arm } i \text{ is pulled}\}$. Since $Y_{s,i}$ are \textit{i.i.d.} samples with sub-Gaussian parameter 1, then 
$Z_{i,t,t^{\prime}} = \frac{\sqrt{|B_1||B_2|}}{\sqrt{|B_1|+|B_2|}} [\frac{1}{|B_1|} \sum_{s\in B_1} Y_{s,i} - \frac{1}{|B_2|} \sum_{s\in B_2} Y_{s,i}]$ is a sub-Gaussian random variable with parameter 1 and mean 0. Therefore,
    \begin{equation*}
    P(|Z_{i,t,t^{\prime}}|> \sqrt{6\log T}) 
         \leq  exp \left(-\frac{{(\sqrt{6\log T})}^2}{2} \right) = \frac{1}{T^3}.
    \end{equation*}
By applying a union bound, we get the probability of a false alarm at a time $t \in [c_{i,j}, c_{i,j+1})$ is $P_f(t)< \frac{1}{T^2}$.
\end{proof}

\subsection{Detection Delay} 
\begin{lem}\label{delay}
In Algorithm \ref{alg:cp_ucb}, define the detection delay for changepoint $c_{i,j}$ of arm $i$ to be $L_{i,j}$, then it satisfies 
\begin{equation}
    L_{i,j} \leq \frac{C^{\prime} \sqrt{T\log T}}{\delta_{i,j}^2}, \text{ w.p. at least } 1- \frac{3}{T}.
\end{equation}
\end{lem}

\begin{proof}
    Assume changepoint $c_{i,j-1}$ has been successfully detected, we are in the process of detecting the next changepoint $c_{i,j}$. 
    For simplicity, define $\mu = \mu_{c_{i,j}-1,i}$ and $\mu^{\prime} = \mu_{c_{i,j},i}$.
    We can assume $\mu - \mu' >0$, then $\delta_{i,j} = \mu - \mu^{\prime} > 0$ (the other case can be proved similarly).
    \begin{eqnarray}\label{equa_orig}
       & & P(|Z_{i,t,t^{\prime}}| > u) \geq  P(Z_{i,t,t^{\prime}} > u) \nonumber\\
       &=& P\left( \tilde Z_{i,t,t^{\prime}} + \sqrt{\frac{|B_1| |B_2|}{|B_1| + |B_2|} } (\mu- \mu') > u \right)
    \end{eqnarray}
        In Equation \ref{equa_orig}, $\tilde Z_{i,t,t^{\prime}} := \sqrt{\frac{|B_1| |B_2|}{|B_1| + |B_2|} } [\frac{1}{|B_1|}\sum_{s \in B_1} (Y_{s,i}-\mu) +  \frac{1}{|B_2|} \sum_{s\in B_2} (\mu' -Y_{s,i})]$. It is easy to see that $\tilde Z_{i,t,t^{\prime}}$ is a sub-Gaussian random variable with parameter 1 and mean 0.
        
        From Assumption \ref{gap}, we know that, $S_{i,j}, S_{i,j+1} \geq 
        \frac{C^{\prime} \sqrt{T\log T} }{{\delta_{i,j}}^2}$, and
        $S_{i,j}\alpha, S_{i,j+1} \alpha \geq \frac{C^{\prime} \log T}{{\delta_{i,j}}^2}$. Define $t= c_{i,j} + \frac{C^{\prime} \sqrt{T\log T}}{{\delta_{i,j}}^2} - 1$. Then we know $t < c_{i,j+1}$, which means the changepoint $c_{i,j+1}$ is behind time $t$ and the interval $[c_{i,j}, t]$ is a stationary period of length $\frac{C^{\prime} \sqrt{T\log T} }{{\delta_{i,j}}^2}$. Choose $t^{\prime} = c_{i,j}$, define
    \begin{align*}
        B_1 &= \{ c_{i,j-1}, \dots, t^{\prime} -1\} \cap \{\text{time arm } i \text{ is pulled}\} \\
        B_2 &= \{t^{\prime},\dots, t\} \cap \{\text{time arm } i \text{ is pulled}\}.
    \end{align*} 
    
    From Lemma \ref{length}, we know that 
    \begin{eqnarray*}
        & & P \left( |B_1| > \frac{C^\prime\log T}{2{\delta_{i,j}}^2}  \right) \geq P \left( |B_1| > \frac{1}{2} S_{i,j} \alpha \right) \\
        &\geq &  1- exp\left( -\frac{1}{8} S_{i,j} \alpha \right) \geq 1- exp\left(- \frac{C^{\prime} \log T}{8{\delta_{i,j}}^2} \right) \\
        &=&  1- \frac{1}{T^{C^{\prime}/(8{\delta_{i,j}}^{2})}} \geq 1-\frac{1}{T}.
    \end{eqnarray*}
    The second step is from Lemma \ref{length}, the last step is due to $\delta_{i,j} \leq M$ and Assumption \ref{gap}.
    
    Since the interval $[t^{\prime}, t] = [c_{i,j}, t]$ is a stationary period of length $\frac{C^{\prime} \sqrt{T\log T} }{{\delta_{i,j}}^2}$, we have the following inequality similarly.
    \begin{eqnarray*}
        & & P \left( |B_2| \geq \frac{C^{\prime} \log T}{2{\delta_{i,j}}^2} \right) =  P \left( |B_2| \geq \frac{1}{2}  \frac{C^{\prime} \sqrt{T\log T}}{{\delta_{i,j}}^2} \alpha \right)\\
        & \geq &  1- exp\left( -\frac{1}{8} \frac{C^{\prime} \sqrt{T\log T}}{{\delta_{i,j}}^2} \alpha \right) =  1- exp\left( -\frac{C^{\prime} \log T}{8{\delta_{i,j}}^2}  \right) \\
        &= & 1- \frac{1}{T^{C^{\prime} /(8{\delta_{i,j}}^{2})}} \geq 1-\frac{1}{T}.
    \end{eqnarray*}
    
     Since $C^{\prime} \geq 24\times (1+\frac{1}{\sqrt{3}})^2$, applying a union bound, we know with probability $(1-\frac{2}{T})$, $\sqrt{\frac{|B_1| |B_2|}{|B_1| + |B_2|} } (\mu- \mu') \geq \sqrt{\frac{C^{\prime}\log T}{4{\delta_{i,j}}^2}} \delta_{i,j} \geq (1+\frac{1}{\sqrt{3}})\sqrt{6 \log T} = (1+\frac{1}{\sqrt{3}}) u$. So from Equation~\ref{equa_orig}, we know that 
        \begin{eqnarray*}
        & & P(|Z_{i,t,t^{\prime}}| > u) \\
        &\geq & P\left( \tilde Z_{i,t,t^{\prime}} > u - \sqrt{\frac{|B_1| |B_2|}{|B_1| + |B_2|} } (\mu- \mu') \right)\\
        &\geq & P\left( \tilde Z_{i,t,t^{\prime}} > - \frac{1}{\sqrt{3}}u \right) \geq  1- exp(-\frac{\frac{1}{3}u^2}{2}) = 1-\frac{1}{T}.
        \end{eqnarray*}
    This means that with probability $1-\frac{3}{T}$, the detection delay for change-point $c_{i,j}$ satisfies $L_{i,j} \leq \frac{C^{\prime} \sqrt{T\log T}}{{\delta_{i,j}}^2}$. 
\end{proof}

\subsection{Proof of Theorem \ref{regret}} \label{proof_thm1}

\begin{proof}
If the changepoints can all be perfectly detected by Algorithm~\ref{alg:cp}, then the total regret $R(T) \leq \sum_{j=1}^{D} R_{\text{UCB}}(S_j) + \sum_{t=1}^T \sum_{i=1}^K \alpha (\mu_{t,*} - \mu_{t,i})]$, where $R_{\text{UCB}}(S_j)$ is the UCB-regret in the $j$-th stationary period ($S_j$ is the length of this stationary period) and the second term is caused by uniform sampling of the arm $i$ with probability $\alpha$ every time. But if changepoints can't be perfectly detected, there may be false alarms and detection delays. So we have 
\begin{eqnarray}
    R(T) &=& \sum_{j=1}^{D} R_{\text{UCB}}(S_j) + \sum_{t=1}^T \sum_{i=1}^K \alpha  (\mu_{t,*} - \mu_{t,i})  \nonumber\\
    & & + R_{\text{false}}(T) + R_{\text{delay}}(T), \label{equ_total}
\end{eqnarray}
where $R_{\text{false}}(T)$ is the regret caused by false alarms up to time $T$, and $R_{\text{delay}}(T)$ is the regret caused by detection delays up to time $T$.

\begin{itemize}
    \item Step 1: Bound the regret from false alarms.
    
    From Lemma $\ref{prob_false_alarm}$, we know that the probability of a false alarm at a time $t \in [c_{i,j-1}, c_{i,j})$ is less than $\frac{1}{T^2}$.
    Define the number of false alarms in the $j$-th stationary period $[c_{i,j-1}, c_{i,j})$ as $f_{i,j}$.
    From union bound, we know that $P(f_{i,j} = 0) \geq 1- \frac{1}{T}$.

    Therefore, the regret caused by false alarm should be 
    \begin{equation}\label{equ_false}
        R_{\text{false}}(T) \leq \sum_{j=1}^{D} R_{\text{UCB}}(S_{j}), 
        \text{ w.p. at least } 1- \frac{1}{T}.
    \end{equation}
    
    \item Step 2: Bound the regret for detection delays.

    If for every arm that changes at changepoint $c_{i,j}$ has $\delta_{i,j} \leq \frac{1}{2} \underline\epsilon$, then the optimal arm doesn't change and there is no regret caused by detection delay. Therefore, we only consider the case where there exists an arm $i$ with $\delta_{i,j} \geq \frac{1}{2} \underline\epsilon$. From Lemma \ref{delay}, we know that the detection delay for the change of this arm satisfies
    $L_{i,j} \leq \frac{C^{\prime} \sqrt{T\log T}}{{\delta_{i,j}}^2}$ w.p. at least $1-\frac{3}{T}$.

    Assume before changepoint $c_{i,j}$, the optimal arm is arm $1$, after changepoint, the optimal arm is arm $2$. The expected reward of arm $1$ changes from $\mu_1$ to $\mu_1^{\prime}$, arm $2$ changes from $\mu_2$ to $\mu_2^{\prime}$. And $\mu_1 > \mu_2$, $\mu_1^{\prime} < \mu_2^{\prime}$. From the definition of $\delta_{i,j}$, we have $\delta_{1,j} = |\mu_1 - \mu_1^{\prime}|$ and $\delta_{2,j} = |\mu_2^{\prime} - \mu_2|$.
    
    \textbf{Case 1}: $\mu_1^{\prime} < \mu_2$. This indicates that the regret of not identifying the optimal arm after changepoint is less than $\delta_{1,j} + \delta_{2,j}$. Also, $\delta_{1,j} \geq \underline \epsilon$.
    
    If $\delta_{1,j} \leq \delta_{2,j}$, This indicates that $\delta_{2,j} \geq \underline\epsilon$.
    Then the regret of detection delay for arm $2$ (not identifying that the optimal arm has changed to arm $2$) is 
    $r \leq L_{2,j} (\delta_{1,j}+\delta_{2,j}) \leq  2 L_{2,j} \delta_{2,j} \leq \frac{2C^{\prime} \sqrt{T\log T}}{{\delta_{2,j}}} \leq \frac{2C^{\prime} \sqrt{T\log T}}{{\max\left(\delta,\underline\epsilon\right)}}$.
    
    If $\delta_{1,j} > \delta_{2,j}$, 
    detecting the change of arm $1$ is enough, since that even if we haven't detected the change for arm 2, detecting that arm $1$ is changed to $\mu_1^{\prime} < \mu_2$ is enough for us to realize that arm $1$ is not optimal after the changepoint.
    Then the regret of this is 
    $r \leq L_{1,j} (\delta_{1,j} + \delta_{2,j}) \leq  2 L_{1,j} \delta_{1,j}  \leq \frac{2C^{\prime} \sqrt{T\log T}}{{\delta_{1,j}}} \leq \frac{2C^{\prime} \sqrt{T\log T}}{{\max\left(\delta,\underline\epsilon\right)}}$.
    
    \textbf{Case 2}: $\mu_1^{\prime} \geq \mu_2$. This indicates that the regret of not realizing the optimal arm changed to arm $2$ is less than $\delta_{2,j}$. Also, $\delta_{2,j} \geq \underline \epsilon$, so the regret of detection delay is 
    $r \leq L_{2,j} \delta_{2,j} \leq \frac{C^{\prime} \sqrt{T\log T}}{{\delta_{2,j}}} \leq \frac{C^{\prime} \sqrt{T\log T}}{{\max\left(\delta,\underline\epsilon\right)}}$.
   
    In conclusion, we have 
    \begin{equation}\label{equ_delay}
        R_{\text{delay}} (T) \leq   \frac{2C^{\prime} D \sqrt{T\log T} }{{\max\left(\delta,\underline\epsilon\right)}}, \text{ w.p. at least } 1- \frac{3}{T}.
    \end{equation}
    
    \item Step 3: Bound the regret for random sampling.
    \begin{equation}\label{equ_random}
        \sum_{t=1}^T \sum_{i=1}^K \alpha  (\mu_{t,*} -\mu_{t,i})   \leq  \bar \epsilon  K \sqrt{T\log T}.
    \end{equation}
    
    Therefore, from Equation \ref{equ_total}, \ref{equ_false}, \ref{equ_delay}, \ref{equ_random}, we have
    with probability at least $1- \frac{4}{T}$,
    \begin{eqnarray*}
       & & R(T) \\
       &\leq & \sum_{j=1}^{D} [2 R_{\text{UCB}}(S_j) + \frac{2C^{\prime} \sqrt{T\log T} }{{\max\left(\delta,\underline\epsilon\right)}}] + \bar \epsilon  K \sqrt{T\log T} \\
        & \sim & O( \frac{D \sqrt{T\log T}}{\max\left(\delta,\underline\epsilon\right)} + \bar\epsilon K\sqrt{T\log T}).
    \end{eqnarray*}
        
\end{itemize}
\end{proof}

\subsection{Experiments for Non-stationary MAB Setting}
We compare our algorithm with state-of-the-art algorithms including Dis-UCB \cite{dis-sw}, SW-UCB \cite{dis-sw}, Cusum-UCB \cite{cusum}, M-UCB \cite{mucb}, Exp3S \cite{exp} and Rexp3 \cite{besbes2014stochastic}. In most of the experiments we conducted, Rexp3 is not as good as other algorithms, so we omit the results of Rexp3 in the plots for better visualization.
In all the experiments in this section, we fix $T=100,000$ and 
we draw sample reward from $N(\mu_{t,i}, 1)$, where $\mu_{t,i}$ is the mean reward for arm $i$ at time $t$. 

We stress here that our algorithm only needs input of $T$, however, all the other algorithms needs input of $D$ (the total number of changepoints) to achieve optimal regret bound. 
To get a fair comparison, we input true $D$ in the experiments for other algorithms. Nevertheless, our algorithm consistently outperforms all the other algorithms. Not to mention that in reality, $D$ usually needs precise training and algorithms usually don't have access to true $D$.

\textbf{Switching Environment:} Following~\cite{cusum}, we consider $K$ arms in switching environment. The reward of arm $i$ at time $t$ is defined to be $\mu_{t,i} = \mu_{t-1,i}$ with probability $1-\frac{D}{T}$ and $\mu_{t,i}$ is a draw from $U[0,1]$ with probability $\frac{D}{T}$. The initial expected reward for each arm is drawn from uniform distribution $U[0,1]$. We repeat the experiments $100$ times to get $100$ different bandit setting, each with randomly chosen expected rewards and changepoints.
We plot the results for $K=100$, $D = 10$ and $K=100$, $D = 6$ in Figure \ref{switch} (first row). 
The true $D$ here is input for all the other algorithms for fair comparison.
The results demonstrate that our proposed algorithm Multiscale-UCB is consistently better than other algorithms. 

\textbf{Flipping Environment:} Following~\cite{cusum}, we run our proposed Multiscale-UCB in flipping environment. Arm 1 is assumed to be stationary all the time with expected reward $\mu_{t,1} = 0.5$. For Arm 2, we assume $\mu_{t,2} = 0.5 - \epsilon$ when $\frac{T}{3} \leq t \leq \frac{2T}{3}$ and $\mu_{t,2} = 0.8$ otherwise. $D=3$ is input for all algorithms for fair comparison.
The results presented here is still the regret averaged from $100$ repeated experiments (the difference of these $100$ experiments lies in the draw of sample rewards, since the expected rewards and changepoints are all fixed in this flipping environment). Shown in Figure \ref{switch} (second row) is the plot for $\epsilon \in \{0.01, 0.06 \}$. 

\begin{figure}[H]
  \begin{center}
    \includegraphics[width=0.5\columnwidth]{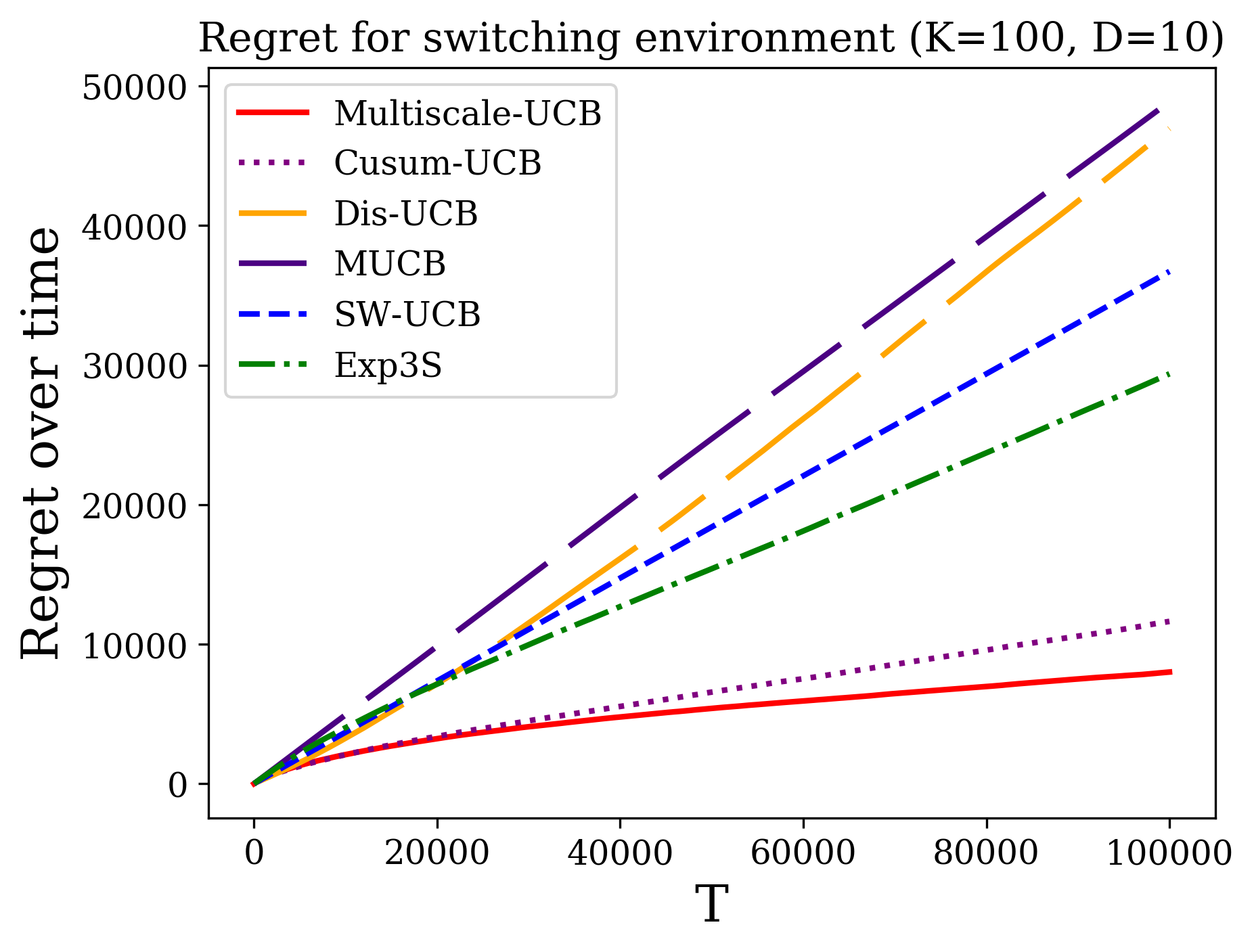}\includegraphics[width=0.5\columnwidth]{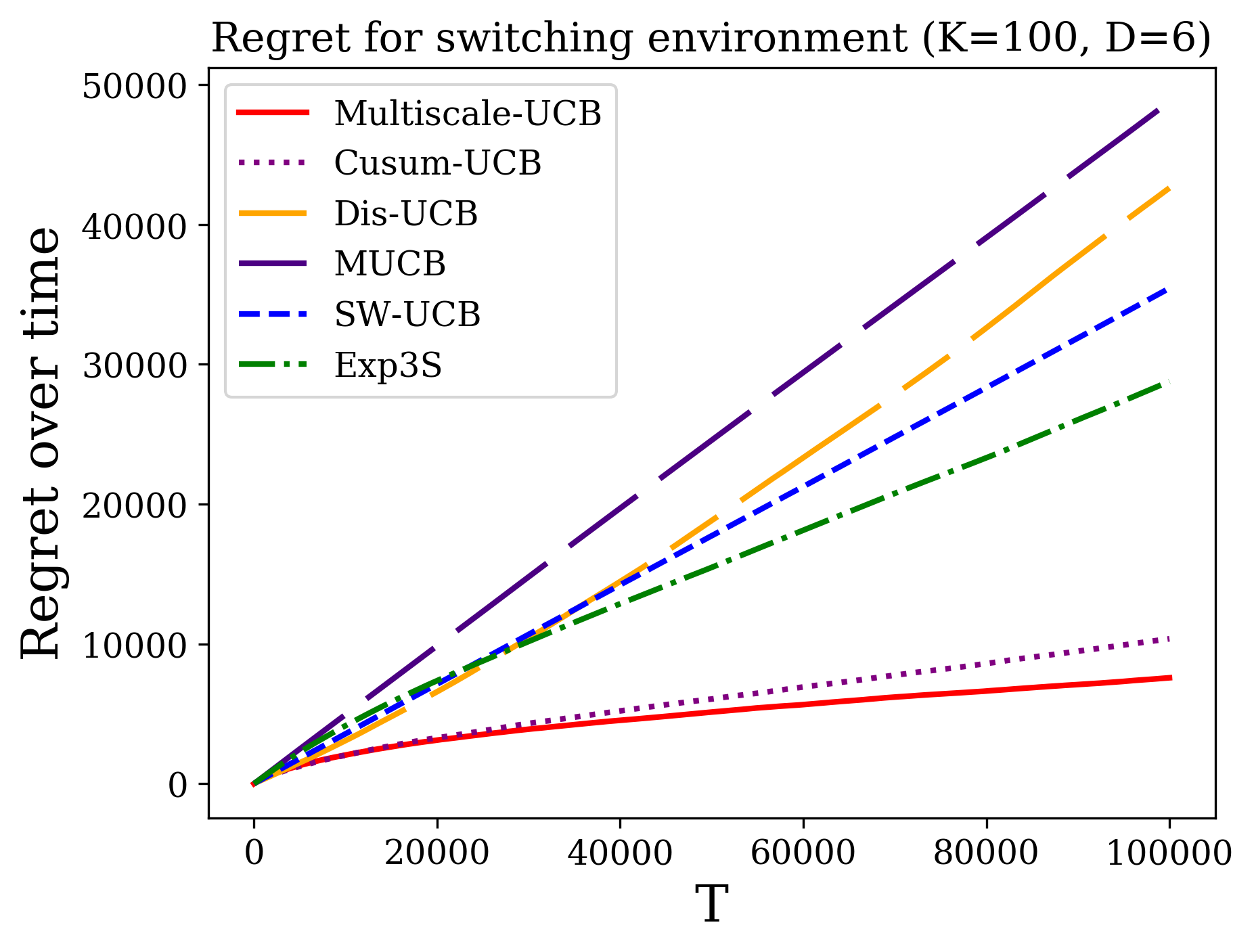}
    
    \includegraphics[width=0.5\columnwidth]{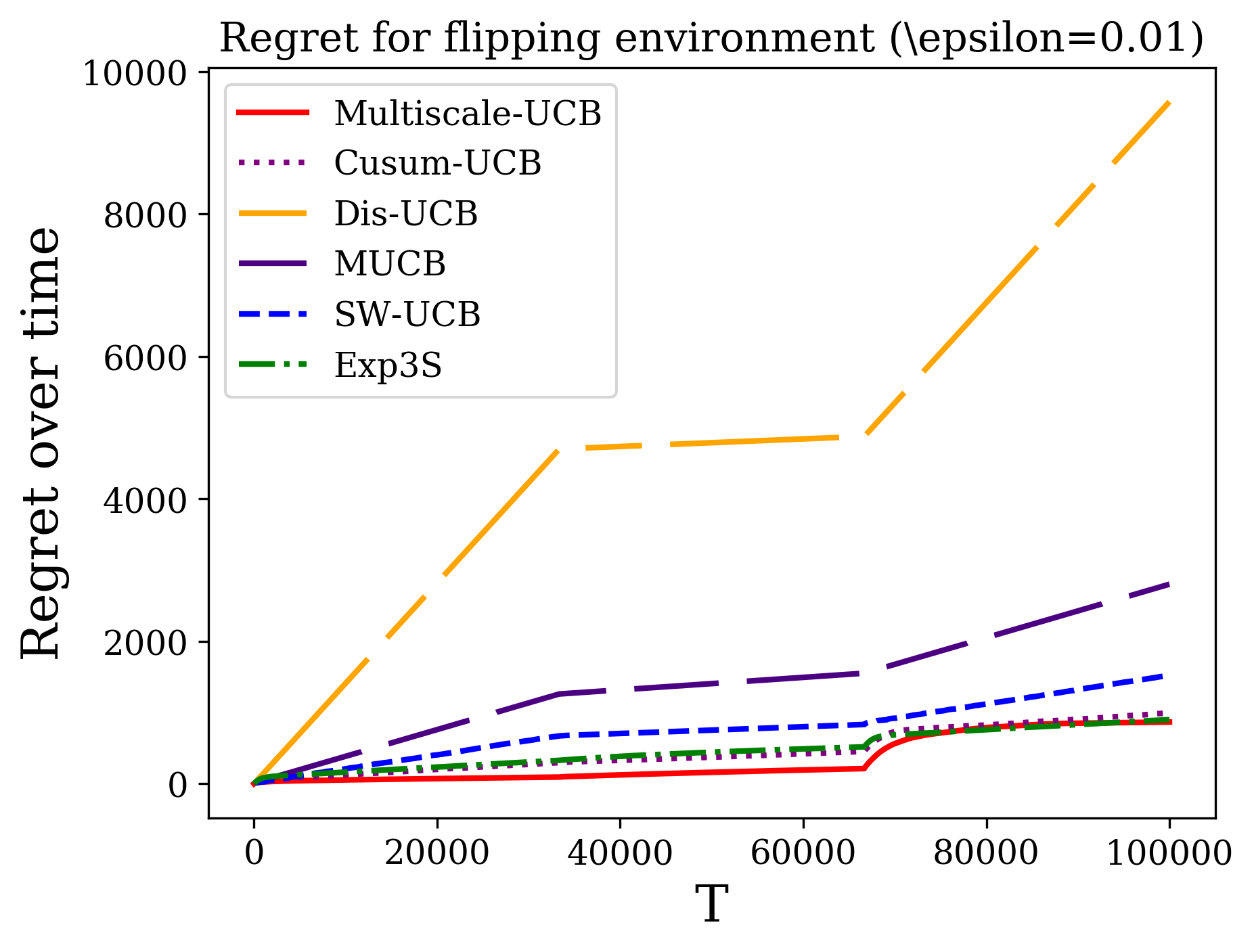}\includegraphics[width=0.5\columnwidth]{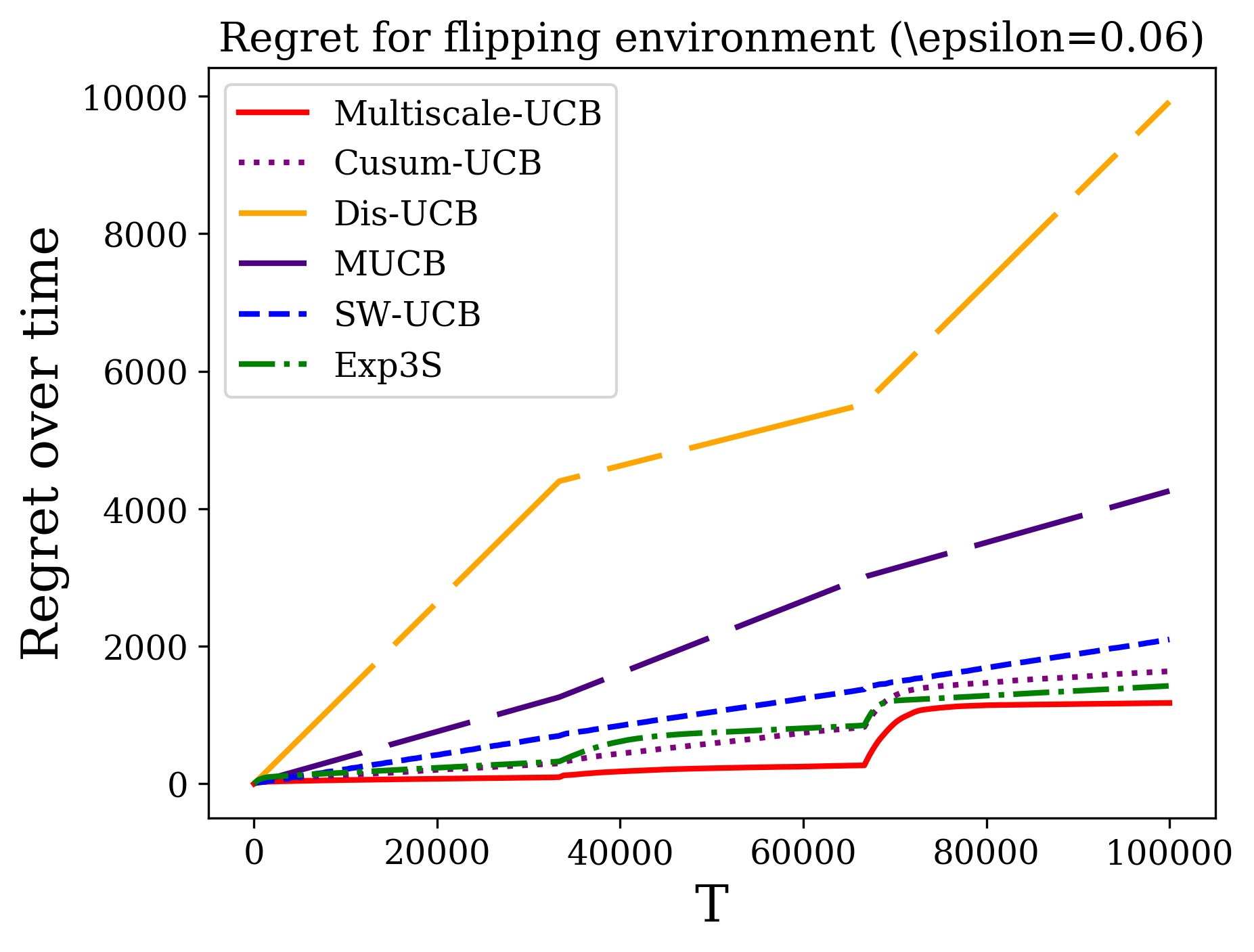}
  \end{center}
  \vskip -0.12in
  \caption{Results for switching environment (first row) and flipping environment (second row).}
  \label{switch}
\end{figure}


\end{document}